\documentclass[11pt,letterpaper]{article}

\usepackage{graphicx}
\usepackage{hyperref}

\usepackage[round]{natbib}

\usepackage[lined,boxed,noend]{algorithm2e}
\newcommand{\myAlgComment}[1]{\tcp{\small\color{blue!60!black} #1}}

\DontPrintSemicolon
\SetKwInOut{Init}{Init}
\SetKwInOut{Given}{Given}

\usepackage{parskip}
\usepackage[letterpaper, left=1in, right=1in, top=1in,
bottom=1in]{geometry}
\usepackage[dvipsnames]{xcolor}
\let\underbar\undefined

\makeatletter
\let\save@mathaccent\mathaccent
\newcommand*\if@single[3]{%
  \setbox0\hbox{${\mathaccent"0362{#1}}^H$}%
  \setbox2\hbox{${\mathaccent"0362{\kern0pt#1}}^H$}%
  \ifdim\ht0=\ht2 #3\else #2\fi
  }
\newcommand*\rel@kern[1]{\kern#1\dimexpr\macc@kerna}
\newcommand*\widebar[1]{\@ifnextchar^{{\wide@bar{#1}{0}}}{\wide@bar{#1}{1}}}
\newcommand*\underbar[1]{\@ifnextchar_{{\under@bar{#1}{0}}}{\under@bar{#1}{1}}}
\newcommand*\wide@bar[2]{\if@single{#1}{\wide@bar@{#1}{#2}{1}}{\wide@bar@{#1}{#2}{2}}}
\newcommand*\under@bar[2]{\if@single{#1}{\under@bar@{#1}{#2}{1}}{\under@bar@{#1}{#2}{2}}}
\newcommand*\wide@bar@[3]{%
  \begingroup
  \def\mathaccent##1##2{%
    \let\mathaccent\save@mathaccent
    \if#32 \let\macc@nucleus\first@char \fi
    \setbox\z@\hbox{$\macc@style{\macc@nucleus}_{}$}%
    \setbox\tw@\hbox{$\macc@style{\macc@nucleus}{}_{}$}%
    \dimen@\wd\tw@
    \advance\dimen@-\wd\z@
    \divide\dimen@ 3
    \@tempdima\wd\tw@
    \advance\@tempdima-\scriptspace
    \divide\@tempdima 10
    \advance\dimen@-\@tempdima
    \ifdim\dimen@>\z@ \dimen@0pt\fi
    \rel@kern{0.6}\kern-\dimen@
    \if#31
      \overline{\rel@kern{-0.6}\kern\dimen@\macc@nucleus\rel@kern{0.4}\kern\dimen@}%
      \advance\dimen@0.4\dimexpr\macc@kerna
      \let\final@kern#2%
      \ifdim\dimen@<\z@ \let\final@kern1\fi
      \if\final@kern1 \kern-\dimen@\fi
    \else
      \overline{\rel@kern{-0.6}\kern\dimen@#1}%
    \fi
  }%
  \macc@depth\@ne
  \let\math@bgroup\@empty \let\math@egroup\macc@set@skewchar
  \mathsurround\z@ \frozen@everymath{\mathgroup\macc@group\relax}%
  \macc@set@skewchar\relax
  \let\mathaccentV\macc@nested@a
  \if#31
    \macc@nested@a\relax111{#1}%
  \else
    \def\gobble@till@marker##1\endmarker{}%
    \futurelet\first@char\gobble@till@marker#1\endmarker
    \ifcat\noexpand\first@char A\else
      \def\first@char{}%
    \fi
    \macc@nested@a\relax111{\first@char}%
  \fi
  \endgroup
}
\newcommand*\under@bar@[3]{%
  \begingroup
  \def\mathaccent##1##2{%
    \let\mathaccent\save@mathaccent
    \if#32 \let\macc@nucleus\first@char \fi
    \setbox\z@\hbox{$\macc@style{\macc@nucleus}_{}$}%
    \setbox\tw@\hbox{$\macc@style{\macc@nucleus}{}_{}$}%
    \dimen@\wd\tw@
    \advance\dimen@-\wd\z@
    \divide\dimen@ 3
    \@tempdima\wd\tw@
    \advance\@tempdima-\scriptspace
    \divide\@tempdima 10
    \advance\dimen@-\@tempdima
    \ifdim\dimen@>\z@ \dimen@0pt\fi
    \rel@kern{0.6}\kern-\dimen@
    \if#31
      \underline{\rel@kern{-0.6}\kern\dimen@\macc@nucleus\rel@kern{0.4}\kern\dimen@}%
      \advance\dimen@0.4\dimexpr\macc@kerna
      \let\final@kern#2%
      \ifdim\dimen@<\z@ \let\final@kern1\fi
      \if\final@kern1 \kern-\dimen@\fi
    \else
      \underline{\rel@kern{-0.6}\kern\dimen@#1}%
    \fi
  }%
  \macc@depth\@ne
  \let\math@bgroup\@empty \let\math@egroup\macc@set@skewchar
  \mathsurround\z@ \frozen@everymath{\mathgroup\macc@group\relax}%
  \macc@set@skewchar\relax
  \let\mathaccentV\macc@nested@a
  \if#31
    \macc@nested@a\relax111{#1}%
  \else
    \def\gobble@till@marker##1\endmarker{}%
    \futurelet\first@char\gobble@till@marker#1\endmarker
    \ifcat\noexpand\first@char A\else
      \def\first@char{}%
    \fi
    \macc@nested@a\relax111{\first@char}%
  \fi
  \endgroup
}
\makeatother

\usepackage{amsthm}
\usepackage{mathtools}
\usepackage{amsmath}
\usepackage{bbm}
\usepackage{amsfonts}
\usepackage{amssymb}
\usepackage{comment}
\usepackage{nicefrac}

\usepackage[capitalize]{cleveref}
\crefname{myAppendix}{Appendix}{Appendices}
\Crefname{myAppendix}{Appendix}{Appendices}

\let\refeq\undefined
\usepackage{slivkins-setup,slivkins-theorems}

\numberwithin{equation}{section}

\usepackage[suppress]{color-edits}  
\addauthor{as}{red}     
\addauthor{df}{ForestGreen}    
\addauthor{ka}{magenta}
\addauthor{xz}{blue}

\DeclarePairedDelimiter{\abs}{\lvert}{\rvert} %
\DeclarePairedDelimiter{\brk}{[}{]}

\DeclarePairedDelimiter{\nrm}{\|}{\|}

\let\Pr\undefined

\DeclareMathOperator{\Pr}{Pr}

\newcommand{\mc}[1]{\mathcal{#1}}

\newcommand{\wt}[1]{\widetilde{#1}}
\newcommand{\wh}[1]{\widehat{#1}}
\newcommand{\wb}[1]{\widebar{#1}}

\def\ddefloop#1{\ifx\ddefloop#1\else\ddef{#1}\expandafter\ddefloop\fi}
\def\ddef#1{\expandafter\def\csname bb#1\endcsname{\ensuremath{\mathbb{#1}}}}
\ddefloop ABCDEFGHIJKLMNOPQRSTUVWXYZ\ddefloop
\def\ddefloop#1{\ifx\ddefloop#1\else\ddef{#1}\expandafter\ddefloop\fi}
\def\ddef#1{\expandafter\def\csname b#1\endcsname{\ensuremath{\mathbf{#1}}}}
\ddefloop ABCDEFGHIJKLMNOPQRSTUVWXYZ\ddefloop
\def\ddef#1{\expandafter\def\csname sf#1\endcsname{\ensuremath{\mathsf{#1}}}}
\ddefloop ABCDEFGHIJKLMNOPQRSTUVWXYZ\ddefloop
\def\ddef#1{\expandafter\def\csname c#1\endcsname{\ensuremath{\mathcal{#1}}}}
\ddefloop ABCDEFGHIJKLMNOPQRSTUVWXYZ\ddefloop
\def\ddef#1{\expandafter\def\csname h#1\endcsname{\ensuremath{\widehat{#1}}}}
\ddefloop ABCDEFGHIJKLMNOPQRSTUVWXYZ\ddefloop
\def\ddef#1{\expandafter\def\csname hc#1\endcsname{\ensuremath{\widehat{\mathcal{#1}}}}}
\ddefloop ABCDEFGHIJKLMNOPQRSTUVWXYZ\ddefloop
\def\ddef#1{\expandafter\def\csname t#1\endcsname{\ensuremath{\widetilde{#1}}}}
\ddefloop ABCDEFGHIJKLMNOPQRSTUVWXYZ\ddefloop
\def\ddef#1{\expandafter\def\csname tc#1\endcsname{\ensuremath{\widetilde{\mathcal{#1}}}}}
\ddefloop ABCDEFGHIJKLMNOPQRSTUVWXYZ\ddefloop
\def\ddefloop#1{\ifx\ddefloop#1\else\ddef{#1}\expandafter\ddefloop\fi}
\def\ddef#1{\expandafter\def\csname scr#1\endcsname{\ensuremath{\mathscr{#1}}}}
\ddefloop ABCDEFGHIJKLMNOPQRSTUVWXYZ\ddefloop

\usepackage{inconsolata} 

\renewcommand{\refeq}[1]{Eq.~(\ref{#1})}

\renewcommand{\xhdr}[1]{\vspace{1mm} \noindent{\bf #1}}

\newcommand{\bigoh}{\mc{O}}
\newcommand{\bigoht}{\wt{\mc{O}}}
\renewcommand{\tO}{\bigoht}

\newcommand{\fhat}{\wh{f}}

\newcommand{\term}[1]{\ensuremath{\texttt{#1}}\xspace}

\newcommand{\eqAND}{\quad\text{and}\quad}

\newcommand{\algTAB}{\text{~~~~}}

\usepackage{framed}
\newenvironment{BoxedProblem}[2][Problem protocol]
    {\begin{oframed}
        {\noindent {\bf #1:} #2}\newline
        \rule{\textwidth}{0.4pt}\vspace{1mm}}
    {\vspace{-1mm}\end{oframed}}

\newcommand{\BwK}{\term{BwK}}
\newcommand{\CBwK}{\term{CBwK}}
\newcommand{\BwLC}{\term{BwLC}} 
\newcommand{\CBwLC}{\term{CBwLC}} 
\newcommand{\CB}{\term{CB}}

\newcommand{\myAlg}{\term{LagrangeCBwLC}} 
\newcommand{\myAlgResc}{\term{LagrangeCBwLC.rescaled}} 

    \newcommand{\LagCBwLC}{\term{LagrangeCBwLC}}

\newcommand{\LagBwK}{\term{LagrangeBwK}}
\newcommand{\SqCB}{\term{SquareCB}}
\newcommand{\UcbBwK}{\term{UCB-BwK}} 

\newcommand{\ALG}{\term{ALG}} 
\newcommand{\OPT}{\term{Opt}}
\newcommand{\rew}{\term{Rew}}

\newcommand{\regOut}{\term{reg}_{\term{out}}} 
\newcommand{\regPace}{\term{reg}_{\term{pace}}} 

\newcommand{\outD}{\mD^{\term{out}}} 
\newcommand{\viol}{V}  
\newcommand{\vM}{\mathbf{M}} 
\newcommand{\vo}{\vec{o}}    
\newcommand{\LagLP}{\mL_{\term{LP}}} 
\newcommand{\OPTLP}{\OPT_{\LP}}

\newcommand{\allPi}{\Pi} 
\newcommand{\oracle}{$\mathrm{\mathbf{Alg}}_{\mathsf{Est}}$\xspace}
\newcommand{\LagOracle}{\mO_{\term{Lag}}} 
\newcommand{\LagU}{U^{\term{Lag}}} 
\newcommand{\aux}{\term{aux}} 
\newcommand{\err}{\mathrm{\mathbf{Err}}} 

\newcommand{\regcontext}{regression-context\xspace} 
\newcommand{\regcontexts}{regression-contexts\xspace}

\newcommand{\Lag}{\mL}       
\newcommand{\dual}[1][t]{\lambda_{#1}}
\newcommand{\PrimalALG}{$\mathrm{\mathbf{Alg}}_{\mathsf{Prim}}$\xspace}
\newcommand{\DualALG}{$\mathrm{\mathbf{Alg}}_{\mathsf{Dual}}$\xspace}
\newcommand{\PrimalREG}{\mathrm{\mathbf{Reg}}_{\mathsf{Prim}}}
\newcommand{\DualREG}{\mathrm{\mathbf{Reg}}_{\mathsf{Dual}}}
\newcommand{\ConcREG}{R_{\mathsf{conc}}}
\newcommand{\RegOne}[1][T]{\widebar{\mathrm{\mathbf{Reg}}}_{\mathsf{Prim}}(#1,\delta)}
\newcommand{\RegTwo}[1][T]{\widebar{\mathrm{\mathbf{Reg}}}_{\mathsf{Dual}}(#1,\delta)}

\newcommand{\cnorm}{c^{\mathrm{norm}}} 
\newcommand{\estL}{\widehat{\Lag}}   

\newcommand{\pbar}{\widebar{D}}

\newcommand{\lambdabarT}{\wb{\lambda}_T}

\newcommand{\enviswitches}{environment-switches\xspace}
\newcommand{\seq}{\vec{\tau}}
\newcommand{\OPTLPt}[1][t]{\OPT_{\term{LP},\,#1}}
\newcommand{\OPTpac}{\OPT_{\term{pace}}}

\newcommand{\LP}{\term{LP}}   

\newcommand{\Vmax}{V_{\max}} 
\newcommand{\LagLPeta}{\LagLP^{\eta}}
\newcommand{\lep}[1]{\mathop  \le \limits^{(#1)}}
\newcommand{\gep}[1]{\mathop  \ge \limits^{(#1)}}
\newcommand{\ep}[1]{\mathop  = \limits^{(#1)}}
\newcommand{\norm}[1]{\left\lVert#1\right\rVert}
\newcommand{\Real}{\mathbb{R}}

\title{Contextual Bandits with Packing and Covering Constraints:\\
A Modular Lagrangian Approach via Regression%
\footnote{
\textbf{Version history.} A preliminary version of this paper, authored by A. Slivkins, K.A. Sankararaman and D.J. Foster, has been published at COLT'23.  \textbf{The current version features an important improvement} (since Jun'24), due to Xingyu Zhou. Specifically, the $\sqrt{T}$-regret result in \cref{thm:main-rootT}(a) holds under a much weaker assumption: that some solution $D$ of the linear program \eqref{eq:LP} is feasible by a constant margin (previously, we needed $D$ to be an \emph{optimal} solution). This improvement is propagated through the corollaries. Consequently, we now position \cref{thm:main-rootT}(a) as the main guarantee, whereas the conference version emphasized the $T^{3/4}$-regret guarantee stemming from the more general analysis in \cref{thm:main-rootT}(b).
Another addition compared to the conference version is a variant of \cref{thm:main-rootT}(a) with zero constraint violation (\cref{app:zero}).

\vspace{2mm}
The Oct'24 version spells out the proof sketches in \cref{sec:shifts}. In particular, we fixed an inaccuracy when the analysis in \cref{sec:Lag} is invoked. Essentially, we observed that this analysis carries over seamlessly in terms of the actual primal (resp., dual) regret rather than the respective regret bounds.

\vspace{2mm}
An earlier version of our result for \CBwK with hard-stopping (\cref{sec:CB} leading up to \cref{cor:CB}(c)) has appeared as a standalone technical report authored by A. Slivkins and D.J. Foster. This technical report has been available at {\tt arxiv.org/abs/2211.07484v1} since November 2022 and was circulated informally since Spring 2022.

\vspace{2mm}
\textbf{Acknowledgements:} The authors are grateful to Giannis Fikioris for insightful conversations on \CBwLC.\newline\vspace{2mm}
}}

\author{%
\begin{tabular}{ccc}
Aleksandrs Slivkins%
\footnote{Microsoft Research NYC. Emails: {\tt\{slivkins,dylanfoster\}@microsoft.com}}
&~~~&
Xingyu Zhou%
\footnote{Wayne State University, {\tt xingyu.zhou@wayne.edu}}
\\
Karthik Abinav Sankararaman%
\footnote{Meta. Email: {\tt karthikabinavs@gmail.com}}
&&
Dylan J. Foster%
\footnotemark[2]
\end{tabular}}

\date{First version: March 2023\\This version: November 2024}

\begin{document}

\maketitle

\begin{abstract}

We consider \emph{contextual bandits with linear constraints} (\CBwLC), a variant of contextual bandits in which the algorithm consumes multiple resources subject to linear constraints on total consumption. This problem generalizes \emph{contextual bandits with knapsacks} (\CBwK), allowing for packing and covering constraints, as well as positive and negative resource consumption. We provide the first algorithm for \CBwLC (or \CBwK) that is based on regression oracles. The algorithm
is simple, computationally efficient, and statistically optimal under mild assumptions.
Further, we provide the first vanishing-regret guarantees for \CBwLC (or \CBwK) that extend beyond the stochastic environment. We side-step strong impossibility results from prior work by identifying a weaker (and, arguably, fairer) benchmark to compare against. Our algorithm builds on \LagBwK \citep{AdvBwK-focs19-conf,AdvBwK-focs19}, a Lagrangian-based technique for \CBwK, and \SqCB \citep{regressionCB-icml20}, a regression-based technique for contextual bandits. Our analysis leverages the inherent modularity of both techniques.


\end{abstract}

\newpage

\section{Introduction}
\label{sec:intro}
\xhdr{Our scope.} We consider a problem called \emph{contextual bandits with linear constraints} (\CBwLC). In this problem, an algorithm chooses from a fixed set of $K$ arms and consumes $d\geq 1$ constrained \emph{resources}. In each round $t$, the  algorithm observes a context $x_t$, chooses an arm $a_t$, receives a reward $r_t\in[0,1]$, and also consumes some bounded amount of each resource. (So, the outcome of choosing an arm is a $(d+1)$-dimensional vector.) The consumption of a given resource could also be negative, corresponding to replenishment thereof. The algorithm proceeds for $T$ rounds, and faces
a constraint on the total consumption of each resource $i$: either a \emph{packing constraint} (``at most $B_i$") or a \emph{covering constraint} (``at least $B_i$") for some parameter $B_i\leq T$. We focus on the stochastic environment, wherein the context and the arms' outcome vectors are drawn from a fixed joint distribution, independently in each round. On a high level, the challenge is to simultaneously handle bandits with contexts and resource constraints.

\CBwLC subsumes two well-studied bandit problems: \emph{contextual bandits}, the special case with no resources, and \emph{bandits with knapsacks} (\BwK), the special case with no contexts \asedit{and some additional simplifications}. Specifically, \BwK is a special case with no contexts,
only packing constraints, non-negative resource consumption, and a \emph{null arm} that allows one to skip a round. Most prior work on \BwK assumes \emph{hard-stopping}: the algorithm must stop (or, alternatively, permanently switch to the null arm) as soon as one of the constraints is violated.%
\footnote{\label{fn:hard-stop}
Hard-stopping may not be feasible without the null arm, and is not meaningful when one has covering constraints (since they are usually not satisfied initially). Further background and references on \BwK can be found in \cref{sec:related}.}
Contextual bandits with knapsacks (\CBwK), a common generalization of contextual bandits and \BwK, has also been explored in prior work.



\asedit{In contextual bandits, even without resources, one typically specifies some additional structure. This is necessary for tractability, both statistical and computational, when one has a large number of possible contexts (as is the case in many/most applications). We adopt one standard approach}
which assumes access to \emph{regression oracle}, a subroutine for solving certain supervised regression problems \citep{regressionCB-icml18,regressionCB-icml20,regressionCB-bypassing}.
\asedit{A contextual bandit algorithm calls a regression oracle to approximate the observed rewards/losses as a function of the corresponding context-arm pairs; this function is used to predict rewards/losses in the future.}%
\footnote{\label{fn:reg-oracle}%
\asedit{Formally, a regression oracle returns a \emph{regression function} that maps context-arm pairs to real values and belongs to some predetermined (and typically simple) function class.}}
This approach is computationally efficient, allows for strong provable guarantees, and tends to be superior in experiments compared to other approaches (see \cref{sec:related}).



\xhdr{Our contributions.}
We design the first algorithm for \CBwLC with regression oracles (in fact, this constitutes the first such algorithm for \CBwK). To handle contexts via the regression-oracle approach, we build on the \SqCB algorithm from \citet*{regressionCB-icml20}. \asedit{\SqCB estimates actions' rewards using the regression function and converts them}
into a distribution over actions that optimally balances exploration and exploitation. To handle resource constraints, we build on the \LagBwK framework of \citet*{AdvBwK-focs19-conf,AdvBwK-focs19}. \LagBwK solves the simpler problem of \BwK \asedit{by setting up} a repeated zero-sum game \asedit{between two bandit algorithms:} the ``primal" algorithm which chooses among arms and the ``dual" algorithm which chooses among resources. The payoffs in this game are given by a natural Lagriangian relaxation \asedit{of the original constrained problem. Note that each of the two algorithms solves a bandit problem without resource constraints (but the payoff distribution changes over time, as it is driven by the other algorithm).}

We make three technical contributions. First, we develop \myAlg, an extension of the \LagBwK framework from \BwK with hard-stopping to \CBwLC. \asedit{The main challenge is to bound constraint violations without hard-stopping (which trivially prevents them). This} necessitates a subtle change in the algorithm (a re-weighting of the Lagrangian payoffs) and some new tricks in the analysis.
\asedit{The framework does not specify a particular primal algorithm, but instead assumes that it satisfies a certain regret bound.
Second, we design a suitable primal algorithm that handles contexts via a regression oracle. This algorithm builds on the \SqCB technique and we formally interpret it as an instance of \SqCB for a suitably defined contextual bandit problem.}
Third, we extend our guarantees beyond stochastic environments, allowing for a bounded number of ``switches" from one stochastic environment to another
(henceforth, the \emph{switching environment}).

We measure performance in terms of 1) regret relative to the best algorithm, and 2) maximum violation of each constraint at time $T$. We bound the maximum of these quantities, henceforth called \emph{outcome-regret}. Our main result attains the (optimal) $\tO(\sqrt{T})$ outcome-regret bound whenever $B>\Omega(T)$ and a minor non-degeneracy assumption holds.
We also attain outcome-regret $\tO(T^{3/4})$ for general \CBwLC problems. We emphasize that these are the first regret bound for \CBwLC with regression oracles. We also obtain $\tO(\sqrt{T})$ regret for contextual \BwK with hard-stopping.


\asmargincomment{Refactored this para}
Our proof leverages the inherent \emph{modularity} of the techniques. A key conceptual contribution here is to identify the pieces and connecting them to one another. In particular,  \myAlg permits the use of any application-specific primal bandit algorithm with a particular regret guarantee,%
\footnote{\asedit{The \LagBwK framework permits a similar modularity for \BwK, and} \citet{AdvBwK-focs19-conf,AdvBwK-focs19} and \citet{Castiglioni-icml22} use this modularity to derive several extensions.}
and our \SqCB-based primal algorithm satisfies this guarantee when it has access to a suitable regression oracle.
%
We provide two ``theoretical interfaces" for \myAlg that a primal algorithm can plug into,
depending on whether the non-degeneracy assumption holds.
We incorporate the original analysis of \SqCB as a theorem which we invoke when analyzing our primal algorithm.
This theorem requires a regression oracle with a particular guarantee on the squared regression error; prior work on regression provides such oracles under various conditions.
This is how our analysis for the stochastic environment comes together.
We then re-use this whole machinery for the analysis of the  switching environment.

\xhdr{Special cases.}
The \myAlg framework is of independent interest for even for the simpler problem of \CBwLC without contexts (henceforth, \BwLC). This is due to two extensions which appear new even without contexts: to the switching environment and to convex optimization (where rewards and resource consumption are convex/concave functions of an arm). \asedit{However, the basic version of \BwLC (\ie stochastic environment without additional structure) was already solved in prior work \citep{AgrawalDevanur-ec14,AgrawalDevanur-ec14-OpRe}, achieving optimal outcome-regret.}



\asedit{Our result for \CBwK with hard-stopping builds on \LagBwK (as a special case of \myAlg). Again, our analysis is modular: we encapsulate prior work on \LagBwK as a theorem that our primal algorithm plugs into.}

Our result for the switching environment is new even for \BwK, \ie when one only has packing constraints and no contexts. This result builds on our analysis for \myAlg: crucially, the algorithm continues till round $T$. Prior analyses of \LagBwK with hard-stopping do not appear to suffice. We obtain regret bounds relative to a non-standard, yet well-motivated benchmark, bypassing strong impossibility results from prior work on Adversarial \BwK (see \cref{sec:related}).



\subsection{Additional background and related work}
\label{sec:related}

Contextual bandits and \BwK generalize (stochastic) multi-armed bandits, \ie the special case without contexts or resource constraints. Further background on bandit algorithms can be found in books \citep{Bubeck-survey12,LS19bandit-book,slivkins-MABbook}.

\xhdr{Contextual bandits (\CB).}
While various versions of the contextual bandit problem have been studied over the past three decades, most relevant are the approaches based on computational oracles.
We focus on \CB with regression oracles, a promising emerging paradigm
\citep{regressionCB-icml18,regressionCB-icml20,regressionCB-bypassing}. \CB with classification oracles is an earlier approach, studied in \citet{Langford-nips07} and follow-up work, \eg \citet{policy_elim,monster-icml14}.
\footnote{A \emph{classification oracle} solves a different problem compared to a regression oracle: it is a subroutine for computing an optimal policy (mapping from contexts to arms) within a given class of policies.}

Contextual bandits with regression oracles are practical to implement, and can leverage the fact that regression algorithms are common in practice. In addition, \CB with regression oracles tend to have superior statistical performance compared to \CB with classification oracles, as reported in extensive real-data experiments \citep{regressionCB-icml18,Foster-colt21,practicalCB-arxiv18}.

\CB with regression oracles are desirable from a theoretical perspective, as they admit \emph{unconditionally} efficient algorithms for various standard function classes under realizability.%
\footnote{\Ie assuming that a given class of regression functions contains one that correctly describes the problem instance.} In contrast, statistically optimal guarantees for \CB with classification oracles are only computationally efficient conditionally. Specifically, one needs to assume that the oracle is an exact optimizer for all possible datasets, even though this is typically an NP-hard problem. This assumption is needed even if the CB algorithm is run on an instance that satisfies realizability.

\emph{Linear} \CB \citep{Langford-www10,Reyzin-aistats11-linear,Csaba-nips11}, a well-studied special case of the regression-based approach to \CB, posits realizability for linear regression functions. Analyses tend to focus on the high-confidence region around regression-based estimates. This variant is less relevant to our paper.

\xhdr{Bandits with Knapsacks (\BwK)} are more challenging compared to stochastic bandits for two reasons. First, instead of \emph{per-round} expected reward one needs to think about the \emph{total} expected reward over the entire time horizon, taking into account the resource consumption. Moreover, instead of the best arm one is interested in the best fixed \emph{distribution} over arms, which can perform much better. Both challenges arise in the ``basic" special case when one has only two arms and only one resource other than the time itself.

The \BwK problem was introduced and optimally solved in \citet{BwK-focs13-conf,BwK-focs13}, achieving $\tO(\sqrt{KT})$ regret for $K$ arms when budgets are $B_i=\Omega(T)$.
\citet{AgrawalDevanur-ec14,AgrawalDevanur-ec14-OpRe} and \citet{AdvBwK-focs19-conf,AdvBwK-focs19} provide alternative regret-optimal algorithms. In particular, the algorithm in \citet{AgrawalDevanur-ec14,AgrawalDevanur-ec14-OpRe}, which we refer to as \UcbBwK, implements the paradigm of \emph{optimism in the face of uncertainty}. Most work on \BwK posits hard-stopping (as defined earlier). A detailed survey of \BwK and its extensions can be found in \citet[Ch.11,][]{slivkins-MABbook}.


The contextual version of \BwK (\CBwK) was first studied in \citet{cBwK-colt14}. They consider \CBwK with classification oracles, and obtain an algorithm that is regret-optimal but not computationally efficient. \citet{CBwK-colt16} provide a regret-optimal and oracle-efficient algorithm for the same problem, which combines \UcbBwK and with the oracle-efficient contextual bandit method \citet{monster-icml14}.
\citet{CBwK-nips16} provide a regression-based approach for the special case of linear \CBwK, combining \UcbBwK and the optimistic approach for linear contextual bandits \citep{Langford-www10,Reyzin-aistats11-linear,Csaba-nips11}. Other regression-based methods for contextual \BwK have not been studied.


Many special cases of \CBwK have been studied for their own sake, most notably dynamic pricing
\citep[\eg][]{BZ09,DynPricing-ec12,Wang-OR14} and online bidding under budget \citep[\eg][]{BalseiroGur19,Balseiro-BestOfMany-Opre,Gaitonde-itcs23}. For the latter, \citet{Gaitonde-itcs23} achieve vanishing regret against a benchmark similar to ours. 

\xhdr{\CBwLC beyond (contextual) \BwK.}
\citet{AgrawalDevanur-ec14,AgrawalDevanur-ec14-OpRe} solve \CBwLC without contexts (\BwLC), building on \UcbBwK and achieving regret $\tO(\sqrt{KT})$. In fact, their result extends to arbitrary  convex constraints and \citep[in][]{CBwK-colt16} to \CBwK with classification oracles. However, their technique does not appear to connect well with regression oracles.

More recently,  \citet{efroni2020exploration,ding2021provably,zhou2022kernelized,ghosh2022provably}
studied various extensions of \BwLC, essentially following \LagBwK framework. They build on the same tools from constrained convex optimization as we do (\eg  Corollary~\ref{cor:Vmax}) in order to bound the constraint violations. However,  they use specific primal and dual algorithms, and their analyses are tailored to these algorithms.%
\footnote{The primal algorithms are based on ``optimistic'' bonuses, and the dual algorithms are based on gradient descent; \citet{zhou2022kernelized} also consider a primal algorithm based on Thompson Sampling.}
In contrast, our meta-theorem allows for arbitrary plug-in algorithms with suitable regret guarantees. \asedit{Moreover, these papers only handle the ``nice" case with Slater's constraint, whereas we also handle the general case. Finally,} the regret bounds in these papers are suboptimal for large $d$, the number of constraints, scaling as $\sqrt{d}$ rather than $\sqrt{\log d}$, even in the non-contextual case.

A notable special case involving covering constraints is online bidding under  return-on-investment constraint
\citep[\eg][]{Balseiro-BestOfMany-Opre,roi-golrezaei2021auction,golrezaei2021bidding}.

The version of \BwK that allows negative resource consumption has not been widely studied. A very recent algorithm in \citet{BwK-drift-neurips22} admits a regret bound that depends on several instance-dependent parameters, but no worst-case regret bound is provided.

\xhdr{Adversarial \BwK.}
The adversarial version of \BwK, introduced in \citet{AdvBwK-focs19-conf,AdvBwK-focs19}, is even more challenging compared to the stochastic version due to the \emph{spend-or-save dilemma}: essentially, the algorithm does not know whether to spend its budget now or to save it for the future. The algorithms are doomed to approximation ratios against standard benchmarks, as opposed to vanishing regret, even for a switching environment with just a single switch \citep{AdvBwK-focs19}. The approximation-ratio version is by now well-understood
\citep{AdvBwK-focs19-conf,AdvBwK-focs19,Singla-colt20,Castiglioni-icml22,Fikioris-BwK23}.%
\footnote{\citet{Fikioris-BwK23} is concurrent and independent work with respect to ours.}
Interestingly, all algorithms in these papers build on versions of \LagBwK. On the other hand, obtaining vanishing regret against some reasonable-but-weaker benchmark (such as ours) is articulated as a major open question \citep{AdvBwK-focs19}. We are not aware of any such results in prior work.

\asedit{\citet{BwK-Liu-neurips22} achieves a vanishing-regret result for Adversarial \BwK against a standard benchmark when one has bounded pathlength and total variation.%
\footnote{In fact, they consider the strongest possible standard benchmark: 
the optimal dynamic policy. The pathlength (or an upper bound thereon) must be known to the algorithm.}
This result is incomparable to our results for the switching environment, which are parameterized by the (unknown) number of switches and holds against a non-standard benchmark. Moreover, our approach extends to contextual bandits with regression oracles, whereas theirs does not.} 



\xhdr{Large vs. small budgets.}
Our guarantees are most meaningful in the regime of ``large budgets", where
$B:=\min_{i\in[d]} B_i>\Omega(T)$.
This is the main regime of interest in all prior work on \BwK and its special cases.%
\footnote{In particular, $B>\Omega(T)$ is explicitly assumed in, \eg   \citet{BZ09,Wang-OR14,BalseiroGur19,Castiglioni-icml22,Gaitonde-itcs23}.} That said, our guarantees are non-trivial even if $B = o(T)$.

The small-budget regime, $B=o(T)$, has been studied since
\citet{DynPricing-ec12}. In particular, \citep{BwK-focs13-conf,BwK-focs13} derive optimal upper/lower regret bounds in this regime for \BwK with hard-stopping. The respective \emph{lower} bounds are specific to hard-stopping, and do not directly apply when a \BwK algorithm can continue till round $T$.




\subsection{Concurrent work}
\citet{SquareCBwK-competition}
focus on \CBwK with hard-stopping in the stochastic environment and obtain a result similar to \cref{cor:CB}(c), also using an algorithm based on \LagBwK and \SqCB. The main technical difference
is that they do not explicitly express their algorithm as an instantiation of \LagBwK, and accordingly do not take advantage of its modularity. Their treatment does not extend to the full generality of \CBwLC, and does not address the switching environment. We emphasize that our results are simultaneous and independent with respect to theirs.

%
%
%

\subsection{Organization}
\cref{sec:prelims} introduces the \CBwLC problem. \cref{sec:alg,sec:Lag} provide our Lagrangian framework (\myAlg) and the associated modular guarantees for this framework.
The material specific to regression oracles is encapsulated in
\cref{sec:CB}, including the setup and the \SqCB-based primal
algorithm. \cref{sec:shifts} extends our results to the switching
environment, defining a novel benchmark and building on the machinery from the previous sections.

\section{Model and preliminaries}
\label{sec:prelims}

\vspace{-2mm}\xhdr{Contextual Bandits with Linear Constraints (\CBwLC).}
There are $K\geq 2$ arms, $T\geq 2$ rounds, and $d\geq 1$ \emph{resources}. We use $[K]$, $[T]$, and $[d]$ to denote, respectively, the sets of all arms, rounds, and resources.%
\footnote{Throughout, $[n]$, $n\in\N$ stands for the set $\{ 1,2 \LDOTS n\}$.}
In each round $t\in[T]$, an algorithm
observes a context $x_t\in\mX$ from a set $\mX$ of possible contexts, chooses an arm $a_t\in[K]$, receives a reward $r_t\in [0,1]$, and consumes some amount $c_{t,i}\in [-1,+1]$ of each resource $i$. Consumptions are observed by the algorithm, so that the outcome of choosing an arm is the \emph{outcome vector}
$\vo_t = (r_t; c_{t,1} \LDOTS c_{t,d})\in [0,1]\times [-1,+1]^d$.
For each resource $i\in[d]$ we are required to (approximately) satisfy the constraint
\begin{align}\label{eq:prelims-constraint}
\viol_i(T) := \textstyle \sigma_i\rbr{\sum_{t\in[T]} c_{t,i} - B_i} \leq 0,
\end{align}
where $B_i\in [0,T]$ is the budget and
    $\sigma_i\in\cbr{-1,+1}$
    is the \emph{constraint sign}. Here $\sigma_i = 1$ (resp., $\sigma_i = -1$) corresponds to a packing (resp., covering) constraint, which requires that the total consumption never exceeds (resp., never falls below) $B_i$.
Informally, the goal is to minimize both regret (on the total reward) and the constraint violations $\viol_i(T)$.
\footnote{A similar bi-objective approach is taken in \citet{AgrawalDevanur-ec14,AgrawalDevanur-ec14-OpRe}  and \citet{CBwK-colt16}.}



We define counterfactual outcomes as follows. The \emph{outcome matrix}
    $\vM_t \in [-1,1]^{K\times (d+1)}$
is chosen in each round $t\in [T]$, so that its rows $\vM_t(a)$ correspond to arms $a\in[K]$ and the outcome vector is defined as
    $\vo_t= \vM_t(a_t)$.
Thus, the row $\vM_t(a)$ represents the outcome the algorithm \emph{would have} observed in round $t$ if it has chosen arm $a$.

We focus on \emph{Stochastic} \CBwLC throughout the paper unless stated otherwise: in each round $t$, the pair $(x_t,\vM_t)$ is drawn independently from some fixed distribution $\outD$. In \cref{sec:shifts} we consider a generalization in which the distribution $\outD$ can change over time.

The special case of \CBwLC without contexts (equivalently, with only one possible context, $|\mX|=1$) is called \emph{bandits with linear constraints} (\BwLC). We also refer to it as the non-contextual problem.


\begin{remark}
Rewards and resource consumptions can be mutually correlated. This is essential in most motivating examples of \BwK, \eg  \cite{BwK-focs13} and \citep[Ch. 10]{slivkins-MABbook}.
\end{remark}

\begin{remark}
We assume i.i.d. context arrivals. While many analyses in contextual bandits seamlessly carry over to adversarial chosen  context arrivals, this is not the case for our problem.%
\footnote{\label{fn:dilemma}
Indeed, with adversarial context arrivals algorithms cannot achieve sublinear regret, and instead are doomed to a constant approximation ratio. To see this, focus on \CBwK and consider a version of the ``spend or save" dilemma from \cref{sec:related}. There are three types of contexts which always yield, resp., high, low, and medium rewards. The contexts are ``medium" in the first $T/2$ rounds, and either all ``high" or all ``low" afterwards. The algorithm would not know whether to spend all its budget in the first half, or save it for the second half.}
In particular, i.i.d. context arrivals are needed to make the linear program \eqref{eq:LP} well-defined.
\end{remark}

\begin{remark}
\BwLC differs from Bandits with Knapsacks (\BwK) in several ways. First, \BwK only allows packing constraints ($\sigma_i\equiv 1$), whereas \BwLC also allows covering constraints ($\sigma_i=-1$). Second, we allow resource consumption to be both positive and negative, whereas on \BwK it must be non-negative. Third, \BwK assumes that some arm in $\brk{K}$ is a ``null arm": an arm with zero reward and consumption of each resource,%
\footnote{Existence of a ``null arm" is equivalent to the algorithm being able to skip rounds.}
whereas \BwLC does not. Moreover, most prior work on \BwK posits \emph{hard-stopping}: the algorithm must stop --- in our terms, permanently  switch to the null arm --- as soon as one of the constraints is violated.
\end{remark}





Let $B = \min_{i\in[d]} B_i$ be the smallest budget. Without loss of generality, we rescale the problem so that all budgets are $B$: we divide the per-round consumption of each resource $i$ by $B_i/B$. 


Without loss of generality, we assume that one of the resources is the \emph{time resource}: it is deterministically consumed by each action at the rate of $B/T$, with a packing constraint $(\sigma_i=1)$.

Formally, an instance of \CBwLC is specified by parameters $T,B,K,d$, constraint signs
    $\sigma_1 \LDOTS \sigma_d$,
and outcome distribution $\outD$. Our benchmark is the best algorithm for a given problem instance: \begin{align}\label{eq:prelims-opt}
\OPT := \sup_{\text{algorithms \ALG with $\E[V_i(T)]\leq 0$ for all resources $i$}}\quad \E[\rew(\ALG)],
\end{align}
where $\rew(\ALG) = \sum_{t\in[T]} r_t$ is the algorithm's total reward (we write $\rew$ when the algorithm is clear from the context). The goal is to minimize algorithm's \emph{regret}, defined as
    $\OPT - \rew(\ALG)$,
as well as constraint violations $\viol_i(T)$. For most lucid results we upper-bound the maximum of these quantities, 
$\regOut := \max_{i\in[d]}\rbr{\OPT - \rew(\ALG),\;\viol_i(T)}$,
called the \emph{outcome-regret}.




\xhdr{Additional notation.}
For the round-$t$ outcome matrix $\vM_t$, the row for arm $a\in[K]$ is denoted
    \[ \vM_t(a) = (r_t(a); \, c_{t,1}(a) \LDOTS c_{t,d}(a)),\]
so that $r_t(a)$ is the reward and $c_{t,i}(a)$ is the consumption of each resource $i$ if arm $a$ is chosen.

Expected reward and resource-$i$ consumption for a given context-arm pair $(x,a)\in\mX\times [K]$ is
    $r(x,a) = \E\sbr{ r_t(a) \mid x_t=x}$
and
    $c_i(x,a) = \E\sbr{ c_{t,i}(a) \mid x_t=x}$,
 where the expectation is over the marginal distribution of $\vM_t$ conditional on $x_t=x$.

A \emph{policy} is a deterministic mapping from contexts to arms. The set of all policies is denoted $\allPi$.  Without loss of generality, we assume that the following happens in each round $t$: the algorithm deterministically chooses some distribution $D_t$ over policies, draws a policy $\pi_t\sim D_t$ independently at random, observes context $x_t$ (\emph{after} choosing $D_t$), and chooses an arm $a_t = \pi_t(x_t)$. The \emph{history} of the first $t$ rounds is $\mH_t := \rbr{D_s,\, \pi_s,\, x_s,\, a_s,\, r_s}_{s\in[t]}$.

Consider a distribution $D$ over policies. Suppose this distribution is ``played" in some round $t$, \ie a policy $\pi_t$ is drawn independently from $D$, and then an arm is chosen as $a_t = \pi_t(x_t)$. The expected reward and resource-$i$ consumption for $D$ are denoted
\begin{align}
r_t(D) := \E_{\pi\sim D}\sbr{ r_t(\pi(x_t))}
    &\text{ and }
c_{t,i}(D) := \E_{\pi\sim D}\sbr{c_{t,i}(\pi(x_t))}
    & \qquad \EqComment{given $(x_t,\vM_t)$}
    \label{eq:prelims-exp-t}\\
r(D) := \E\sbr{ r_t(\pi(x_t))}
    &\text{ and }
c_i(D) := \E\sbr{c_{t,i}(\pi(x_t))},
    \label{eq:prelims-exp}
\end{align}
where the expectation is over both policies $\pi\sim D$ and context-matrix pairs  $(x_t,\vM_t)\sim\outD$. Note that
    $\E\sbr{r_t(a_t) \mid \mH_{t-1}} = r(D_t)$
and likewise
    $\E\sbr{c_{t,i}(a_t) \mid \mH_{t-1}} = c_i(D_t)$.
If distribution $D$ is played in all rounds $t\in[T]$, the expected constraint-$i$ violation is denoted as
    $V_i(D) := \sigma_i\rbr{T\cdot c_i(D)-B}$.

A policy can be interpreted as a singleton distribution that chooses this policy almost surely, and an arm can be interpreted as a policy that always chooses this arm. So, the notation in \cref{eq:prelims-exp-t,eq:prelims-exp} can be overloaded naturally to input a policy, an arm, or a distribution over arms.%
\footnote{For the non-contextual problem, for example, $r(a)$ is the expected reward of arm $a$, and $r(D) = \E_{a\in D}\sbr{r(a)}$ is the expected reward for a distribution $D$ over arms.}



%
%
%
%
%
%

Let $\Delta_S$ denote the set of all distributions over set $S$. We write $\Delta_n = \Delta_{[n]}$ for $n\in\N$ as a shorthand. We identify $\Delta_K$ (resp., $\Delta_d$) with the set of all distributions over arms (resp., resources).




\xhdr{Linear relaxation.}
We use a standard linear relaxation of \CBwLC, which optimizes over distributions over policies,
    $D \in \Delta_{\allPi}$,
maximizing the expected reward $r(D)$ subject to the constraints:
\begin{equation}
\label{eq:LP}
\begin{array}{ll}
\text{maximize}
    &r(D)\\
\text{subject to} &D \in \Delta_{\allPi}  \\
    & V_i(D) := \sigma_i\rbr{ T\cdot c_i(D) - B}\leq 0 \qquad \forall i \in [d].
\end{array}
\end{equation}
\noindent The value of this linear program is denoted $\OPTLP$. It is easy to see that
    $T\cdot \OPTLP= \OPT$.
\footnote{Indeed, to see that $T\cdot \OPTLP\geq \OPT$, consider any algorithm in the supremum in \refeq{eq:prelims-opt}. Let $D_\pi$ be the expected fraction of rounds in which a given policy $\pi\in\allPi$ is chosen. Then distribution $D \in\Delta_{\allPi}$ satisfies the constraints in the LP. For the other direction, consider an LP-optimizing distribution $D$ and observe that using this distribution in each round constitutes a feasible algorithm for the benchmark \refeq{eq:prelims-opt}.}

The Lagrange function associated with the linear program \eqref{eq:LP} is defined as follows:
\begin{align}\label{eq:Lag}
\textstyle
\LagLP(D, \lambda)
:=  r(D)\;+\;
    \sum_{i\in [d]}\; \sigma_i\cdot\lambda_i  \rbr{1-\tfrac{T}{B}\;  c_i(D) },
    \quad D\in\Delta_{\allPi},\, \lambda\in \R_+^d.
\end{align}
A standard result concerning \emph{Lagrange duality} states that the maximin value of $\LagLP$ coincides with $\OPTLP$. For this result, $D$ ranges over all distributions over policies and $\lambda$ ranges over all of $\R_+^d$:
\begin{align}\label{eq:Lag-duality-gen}
\OPTLP
    = \sup_{D\in\Delta_{\allPi}} \inf_{\lambda \in\R_+^d} \LagLP(D,\lambda).
\end{align}



\section{Lagrangian framework for \CBwLC}
\label{sec:alg}

We provide a new algorithm design framework, \myAlg, which generalizes the \LagBwK framework from \citet{AdvBwK-focs19-conf,AdvBwK-focs19}. We consider a repeated zero-sum game between two algorithms: a \emph{primal algorithm} \PrimalALG that chooses arms $a\in[K]$, and a \emph{dual algorithm} \DualALG that chooses distributions $\lambda\in\Delta_d$ over resources;%
\footnote{The terms `primal' and `dual' here refer to the duality in linear programming. For the LP-relaxation \eqref{eq:LP}, primal variables correspond to arms, and dual variables (\ie variables in the dual LP) correspond to resources.} \DualALG goes first, and \PrimalALG can react to the chosen $\lambda$. The round-$t$ payoff (\emph{reward} for \PrimalALG, and \emph{cost} for \DualALG) is defined as
\begin{align}\label{eq:Lag-t}
\textstyle
\Lag_t(a,\lambda) = r_t(a) \;+\;
        \eta\cdot \sum_{i\in[d]} \sigma_i\cdot\lambda_i  \rbr{1 - \tfrac{T}{B}\, c_{t,i}(a)}.
\end{align}
Here, $\eta\geq 1$ is a parameter specified later. For a distribution over policies, $D\in\Delta_{\allPi}$,  denote
    $\Lag_t(D,\lambda) = \E_{\pi\sim D}\sbr{\Lag_t(\pi(x_t),\,\lambda)}$.
The purpose of the definition \refeq{eq:Lag-t} is to ensure that
\begin{align}\label{eq:Lag-exp}
\E\sbr{ \Lag_t(D,\lambda)} = \LagLP(D,\eta\cdot \lambda),
\end{align}
where the expectation is over the context $x_t$ and the outcome $\vo_t$.  The repeated game is summarized in \cref{alg:LagBwK}.



\begin{algorithm}[!h]
\Given{$K$ arms, $d$ resources, and ratio $T/B$, as per the problem definition;\\
    parameter $\eta\geq 1$; algorithms \PrimalALG, \DualALG.}
\For{rounds $t\in [T]$}{
Dual algorithm \DualALG outputs a distribution $\lambda_t\in \Delta_d$ over resources.\\
Primal algorithm \PrimalALG receives $(x_t,\lambda_t)$ and outputs an arm $a_t\in [K]$. \\
Arm $a_t$ is played and outcome vector $\vo_t$ is observed (and passed to both algorithms).\\
Lagrange payoff $\Lag_t(a_t,\lambda_t)$ is computed as per \refeq{eq:Lag-t},\\
\algTAB and reported to \PrimalALG as reward and \DualALG as cost.
} 
\caption{\myAlg framework}
\label{alg:LagBwK}
\end{algorithm}

\begin{remark}
Beyond incorporating contexts, the main change compared to \LagBwK \citep{AdvBwK-focs19-conf,AdvBwK-focs19} is that we scale the constraint terms in the Lagrangian by the parameter $\eta\geq 1$. This parameter is the ``lever" that allows us as to extend the algorithm from \BwK to \BwLC, accommodating general constraints. This modification effectively rescale the dual vectors from distributions $\lambda\in\Delta_d$ to vectors $\eta\cdot\lambda\in\bbR^d_+$. An equivalent reformulation of the algorithm could instead rescale all \emph{rewards} to lie in the interval $[0,\nicefrac{1}{\eta}]$. This reformulation is instructive because the scale of rewards can be arbitrary as far as the original problem is concerned, but it leads to some notational difficulties in the analysis, which is why we did not choose it for presentation.
Interestingly, setting $\eta=1$, like in \citep{AdvBwK-focs19-conf,AdvBwK-focs19}, does not appear to suffice even for \BwK if hard-stopping is not allowed (i.e., \cref{alg:LagBwK} must continue as defined till round $T$).

Lastly, we mention two further changes compared to \LagBwK: we allow \PrimalALG to respond to the chosen $\lambda_t$, which is crucial to handle contexts in Section~\ref{sec:CB}, and we rescale the time consumption in \cref{thm:main-rootT}, which allows for improved regret bounds.
\end{remark}

\begin{remark}
A version of \LagBwK with parameter $\eta = \nicefrac{T}{B}$ was recently used in \citet{Castiglioni-icml22}. Their analysis is specialized to \BwK and targets (improved) approximation ratios for the adversarial version. An important technical difference is that their algorithm does not make use of the \emph{time resource}, a dedicated resource that track the time consumption.
\end{remark}

\begin{remark}
  In \myAlg, the dual algorithm \DualALG receives \emph{full feedback} on its Lagrange costs: indeed, the outcome vector $\vo_t$ allows \cref{alg:LagBwK} to reconstruct
    $\Lag_t(a_t,i)$
for each resource $i\in[d]$.
\DualALG could also receive the context $x_t$, but our analysis does not make use of this.
\end{remark}

The intuition behind \myAlg is as follows. If \PrimalALG and \DualALG satisfy certain regret-minimizing properties, the repeated game converges to a Nash equilibrium for the rescaled Lagrangian $\LagLP(D,\eta\cdot \lambda)$.
The specific definition \eqref{eq:Lag-t}, for an appropriate choice of $\eta$, ensures that  the strategy of \PrimalALG in the Nash equilibrium is (near-)optimal for the problem instance by a suitable version of Lagrange duality. For \BwK with hard-stopping problem, $\eta=1$ suffices,%
\footnote{Because
    $\OPTLP = \sup_{D\in\Delta_K} \inf_{\lambda\in\Delta_d} \LagLP(D,\lambda)$
when $\sigma_i\equiv 1$ and there is a null arm \citep{AdvBwK-focs19-conf,AdvBwK-focs19}.}
but for general instances of \BwLC we choose $\eta>1$ in a fashion that depends on the problem instance.

\xhdr{Primal/dual regret.}
We provide general guarantees for \myAlg when invoked with arbitrary primal and dual algorithms \PrimalALG and \DualALG satisfying suitable regret bounds. We define the \emph{primal problem} (resp., \emph{dual problem}) as the online learning problem faced by \PrimalALG (resp., \DualALG) from the perspective of the repeated game in \myAlg. The primal problem is a bandit problem where algorithm's action set is the set of all arms, and the Lagrange payoffs are rewards. The dual problem is a full-feedback online learning problem where algorithm's ``actions" are the resources in \CBwLC, with Lagrange payoffs are costs. The \emph{primal regret} (resp., \emph{dual regret}) is the regret relative to the best-in-hindsight action in the respective problem. Formally, these quantities are as follows:
\begin{align}
\PrimalREG(T) &:= \textstyle
    \sbr{\max_{\pi\in\allPi} \sum_{t\in[T]} \Lag_t(\pi(x_t),\lambda_t )}
    -
    \sum_{t\in[T]} \Lag_t(a_t,\lambda_t ). \nonumber \\
\DualREG(T) &:= \textstyle
    \sum_{t\in[T]} \Lag_t(a_t,\lambda_t )
    -
    \sbr{\min_{i\in[d]} \sum_{t\in[T]} \Lag_t(a_t,i)}.
    \label{eq:primal-dual-regret-defn}
\end{align}
We assume that the algorithms under consideration provide high-probability upper bounds on the primal and dual regret:
\begin{align}\label{eq:regret-assn}
\Pr\sbr{\forall \tau\in[T]\quad
    \PrimalREG(\tau) \leq \RegOne[\tau] \quad\text{and}\quad \DualREG(\tau) \leq \RegTwo[\tau]
} \geq  1-\delta,
\end{align}
where $\RegOne$ and $\RegTwo$ are functions, non-decreasing in $T$, and $\delta\in(0,1)$ is the failure probability. Our theorems use ``combined" regret bound $R(T,\delta)$ defined by
\begin{align}\label{eq:regret-joint}
\nicefrac{T}{B}\cdot\eta\cdot R(T,\delta)
    := \RegOne+\RegTwo+ 2\,\ConcREG(T, \delta),
\end{align}
where $\ConcREG(T, \delta) = O\rbr{\asedit{\nicefrac{T}{B}\cdot\eta}\cdot\sqrt{T\log (dT/\delta)}}$ accounts for concentration.

\begin{remark}\label{rem:standard}
The range of Lagrange payoffs is proportional to $\nicefrac{T}{B}\cdot\eta$, which is why we separate out this factor on the left-hand side of \refeq{eq:regret-joint}. For the non-contextual version with $K$ arms, standard results yield
    $R(T,\delta) =  O(\sqrt{KT\log (dT/\delta)}\;)$.%
\footnote{Using algorithms EXP3.P \citep{bandits-exp3} for \PrimalALG and Hedge \citep{FS97} for  \DualALG, we obtain \refeq{eq:regret-assn} with
$\RegOne = O(\;\nicefrac{T}{B}\cdot\eta\cdot\sqrt{KT \log(K/\delta)}\;)$
and
$\RegTwo = O(\;\nicefrac{T}{B}\cdot\eta\cdot\sqrt{T \log d}\;)$.}
Several other applications of \LagBwK framework (and, by extension, of \LagCBwLC) are discussed in \citep{AdvBwK-focs19,Castiglioni-icml22}. In \cref{sec:CB}, we provide a new primal algorithm for \CBwLC with regression oracles.
Most applications, including ours, feature $\tO(\sqrt{T})$ scaling for $R(T,\delta)$.
\end{remark}

\begin{remark}
For our results, \eqref{eq:regret-joint} with $\tau=T$ suffices. We only use the full power of \refeq{eq:regret-joint} to incorporate the prior-work results on \CBwK with hard-stopping, \ie \cref{thm:main-BwK} and its corollaries.
\end{remark}

\xhdr{Our guarantees.}
 Our main guarantee for \LagCBwLC holds whenever some solution for the LP~\eqref{eq:LP} is feasible \emph{by a constant margin}. Formally, a distribution $D\in\Delta_{\allPi}$ is called \emph{$\zeta$-feasible}, $\zeta\in [0,1)$  if for each non-time resource $i\in[d]$ it satisfies
    $\sigma_i\rbr{\tfrac{T}{B}\, c_i(D)-1} \leq -\zeta$,
and we need some $D$ to be $\zeta$-feasible with $\zeta>0$.
This is a very common assumption in convex analysis, known as \emph{Slater's condition}. It holds without loss of generality when all constraints are packing constraints ($\sigma_i\equiv 1$) and there is a null arm (\ie it is feasible to do nothing); it is a mild ``non-degeneracy" assumption in the general case of \BwLC. We obtain, essentially, the best possible guarantee when Slater's condition holds and $B>\Omega(T)$. We also obtain a non-trivial (but weaker) guarantee when $\zeta=0$, \ie we are only guaranteed \emph{some} feasible solution. Importantly, the $\zeta$-feasible solution is only needed for the analysis: our algorithm does not need to know it (but it does need to know $\zeta$).


\begin{theorem}\label{thm:main}
Suppose some solution for LP~\eqref{eq:LP} is $\zeta$-feasible, for a known margin $\zeta\geq 0$. Fix some $\delta>0$ and consider the setup in \cref{eq:primal-dual-regret-defn,eq:regret-assn,eq:regret-joint}
with ``combined" regret bound $R(T,\delta)$.

\begin{itemize}
\item[(a)] \label{thm:main-rootT}
If $\zeta>0$, consider \myAlg with parameter $\eta = 2/\zeta$. With probability at least $1-O(\delta)$,
\begin{align} 
\max_{i\in[d]}\rbr{\OPT-\rew,\; V_i(T)}
    \leq O\rbr{ \nicefrac{T}{B}
                    \cdot\nicefrac{1}{\zeta}\cdot R(T,\delta) }.
\end{align}

\item[(b)]
Consider \myAlg with parameter $\eta=\frac{B}{T}\,\sqrt{\frac{T}{R(T,\delta)}}$.
With probability at least $1-O(\delta)$,
\begin{align} \label{eq:thm:main-0}
\max_{i\in[d]}\rbr{\OPT-\rew,\; V_i(T)}
    \leq O\rbr{\sqrt{T\cdot R(T,\delta)}}.
\end{align}
\end{itemize}
\end{theorem}



The guarantee in part (a) is the best possible for \BwLC, in the regime when $B>\Omega(T)$ and $\zeta$ is a constant. To see this, consider non-contextual \BwLC with $K$ arms: we obtain  regret rate $\tO(\sqrt{KT})$ by Remark~\ref{rem:standard}. This regret rate is the best possible in the worst case, even without resource constraints, due to the lower bound in \citep{bandits-exp3}. However, our guarantee is suboptimal when $B=o(T)$, compared to $\tO(\sqrt{KT})$ outcome-regret achieved in \citet{AgrawalDevanur-ec14,AgrawalDevanur-ec14-OpRe} via a different approach.

To characterize the regret rate in part (b), consider the paradigmatic regime when
    $R(T,\delta) = \tO(\sqrt{\Psi\cdot T})$
for some parameter $\Psi$ that does not depend on $T$.
\footnote{Here and elsewhere, $\tO\rbr{\cdot}$ notation hides $\log(Kdt/\delta)$ factors when the left-hand side depends on both $T$ and $\delta$.}
Then the right-hand side of \refeq{eq:thm:main-0} becomes
    $\tO\rbr{ \Psi^{\nicefrac{1}{4}}\cdot T^{3/4}}$.

\begin{remark}\label{rem:zero}
When $B>\Omega(T)$ and the margin $\zeta>0$ is a constant, the guarantee in \cref{thm:main}(a) can be improved to zero constraint violation, with the same regret rate.%
\footnote{\asedit{Here, zero constraint violation happens with high probability rather than almost surely. Some recent work on the special case of online bidding under constraints
\citep[\eg][]{Feng-www23,Autobidding-colt24} 
achieves zero constraint violation almost surely
even when the algorithm must continue till time $T$.
}}
The algorithm is modified slightly, by rescaling the budget parameter $B$ and the consumption values $c_{t,i}$ passed to the algorithm, and one need to account for this rescaling in the analysis. See \cref{app:zero} for details.
\end{remark}

The two guarantees in \cref{thm:main} can be viewed as ``theoretical interfaces" to \LagCBwLC framework. We obtain them as special cases of a more general analysis (\cref{thm:main-detailed}), which is deferred to the next section. The main purpose of these guarantees is to enable applications to regression oracles (\cref{sec:CB}) and also to the switching environment (\cref{sec:shifts}, via \cref{thm:main-detailed}). An additional application --- to bandit convex optimization, which may be of independent interest --- is spelled out in \cref{app:BCO}.

For the last result in this section, we restate another ``theoretical interface", which concerns the simpler \CBwK problem and gives an $\nicefrac{T}{B}\cdot R(T,\delta)$ regret rate with parameter $\eta=1$ whenever hard-stopping is allowed.%
\footnote{Recall that under hard-stopping the algorithm effectively stops as soon as some constraint is violated, and therefore all constraint violations are bounded by $1$.}
We invoke this result in \cref{sec:CB} along with \cref{thm:main}.
\begin{theorem}[\citet{AdvBwK-focs19-conf,AdvBwK-focs19}]\label{thm:main-BwK}
Consider \CBwK with hard-stopping. Fix some $\delta>0$ and consider the setup in \cref{eq:primal-dual-regret-defn,eq:regret-assn,eq:regret-joint}. Consider algorithm \myAlg with parameter $\eta=1$. With probability at least $1-O(\delta)$, we have
$\OPT-\rew \leq  \nicefrac{T}{B}\cdot O(R(T,\delta))$.
\end{theorem}

\section{Analysis of \myAlg}
\label{sec:Lag}

We obtain \cref{thm:main} in a general formulation: with an arbitrary choice for the parameter $\eta$ (which we tune optimally to obtain \cref{thm:main}) and with realized primal/dual regret (rather than upper bounds thereon). The latter is needed to handle the switching environment in \cref{sec:shifts}.



\begin{theorem}\label{thm:main-detailed}
Suppose some solution for LP~\eqref{eq:LP} is $\zeta$-feasible, for a known margin $\zeta\in [0,1)$. Run \myAlg with some primal/dual algorithms and parameter $\eta \ge 1$; write $\eta' = \eta \cdot T/B$ as shorthand. Fix some $\delta>0$ and denote 
\begin{align}\label{eq:regret-joint-realized}
R:= \rbr{\PrimalREG(T)+\DualREG(T)+ 2\,\ConcREG(T, \delta)}/\eta'.
\end{align}
Then:
\begin{itemize}
\item[(a)] For $\zeta >0 $ and any $\eta \ge \frac{2}{\zeta}$,
\begin{align*}
  \Pr\sbr{\OPT-\rew \le 3\eta' R \;\text{ and }\;
    \max_{i\in[d]}\,V_i(T) \le  4R + \eta' R} \geq 1-O(\delta).
\end{align*}

\item[(b)] For $\zeta = 0$ and any $\eta\geq 1$,
\begin{align*}
   \Pr\sbr{\OPT-\rew  \le 3\eta' R \;\text{ and }\;
   \max_{i\in[d]}\,V_i(T) \le \frac{T}{\eta'} + 2R + \eta' R} \geq 1-O(\delta).
\end{align*}
\end{itemize}
\end{theorem}



Note that our regret bound $\OPT-\rew \le 3\eta' R$ holds for any $\zeta\geq 0$. We use $\zeta>0$ to provide a sharper bound on constraint violations.

\subsection{Tools from Optimization (for the proof of \cref{thm:main-detailed})}



Our proof builds on some techniques from prior work on linear optimization. When put together, these techniques provide a crucial piece of the overall argument.

Specifically, we formulate two lemmas that connect approximate saddle points, Slater's condition, and the maximal constraint violation for a given distribution $D\in\Delta_\Pi$,
\begin{align}\label{eq:Vmax}
\Vmax(D) := \max_{i\in[d]} \sbr{\sigma_i\rbr{ T\cdot c_i(D) - B}}_+.
\end{align}
Using the notation from \refeq{eq:Lag}, let us define
\begin{align}
\label{eq:lang}
   \LagLPeta (D, \lambda) :=  \LagLP(D, \eta\lambda).
\end{align}
%
%
%

\noindent The first lemma is on the properties of approximate saddle points of $\LagLPeta$. A \emph{$\nu$-approximate saddle point} of $\LagLPeta$ is a pair
    $(D',\lambda')\in \Delta_\Pi \times \Delta_d$
such that
\begin{align*}
 &\LagLPeta(D', \lambda') \ge \LagLPeta(D, \lambda') - \nu,\quad \forall D \in \Delta_{\Pi} \\
 &\LagLPeta(D', \lambda') \le \LagLPeta(D', \lambda) + \nu, \quad \forall \lambda \in \Delta_d.
\end{align*}

\begin{lemma}
\label{lem:app-saddle}
Let    $(D',\lambda')$ be a $\nu$-approximate saddle point  of $\LagLPeta$.
Then it satisfies the following properties for any feasible solution $D\in\Delta_\Pi$ of LP in~\eqref{eq:LP}:
\begin{itemize}
    \item[(a)] $r(D') \ge r(D) - 2\nu$.
    \item[(b)] $r(D) - r(D') + \frac{\eta}{B}\, \Vmax(D')
    \le 2\nu$.
\end{itemize}
\end{lemma}

\cref{lem:app-saddle} follows from Lemmas 2-3 in~\cite{agarwal2018reductions};
we provide a standalone proof in \cref{app:lem:app-saddle} for completeness.

The second lemma is on bounding the constraint violation under Slater's condition (\ie $\zeta > 0$).

\begin{lemma}[Implications of Slater's condition]
\label{lem:Slater}
Consider the linear program in~\eqref{eq:LP} and suppose Slater's condition holds, i.e., some distribution $\hat{D} \in \Delta_{\Pi}$ is $\zeta$-feasible, $\zeta>0$. Suppose for some numbers $C \ge 2/\zeta$, $\gamma > 0$ and distribution $\widetilde{D} \in \Delta_{\Pi}$ the following holds:
 \begin{align*}
          r(D^*) - r(\widetilde{D}) +
          \tfrac{C}{B}\, \Vmax(\widetilde{D})
          \le \gamma,
    \end{align*}
where $D^*$ is an optimal solution of~\eqref{eq:LP}.
    Then
    $\frac{C}{B}\; \Vmax(\widetilde{D}) \leq 2\gamma$.
\end{lemma}



Similar results have appeared in~\cite[Theorem 42 and Corollary 44]{efroni2020exploration}, which are variants of results in~\cite[Theorem 3.60 and Theorem 8.42]{beck2017first}, respectively. For completeness, we provide a standalone proof
in Appendix~\ref{app:lem:Slater}.

We use \cref{lem:app-saddle}(b) and \cref{lem:Slater} through the following corollary.%

\begin{corollary}\label{cor:Vmax}
Suppose some distribution over policies is $\zeta$-feasible, for some $\zeta\geq 0$. Let $(D',\lambda')$ be a $\nu$-approximate saddle point  of $\LagLPeta$. Then, even for $\zeta=0$, we have
\begin{align}
\label{eq:cons-zeta-0}
    \frac{\eta}{B}\; \Vmax(D') \leq 2\nu+1.
\end{align}
Moreover, if $\zeta>0$, then a sharper bound is possible:
\begin{align}
\label{eq:cons-zeta-p}
    \frac{\eta}{B}\; \Vmax(D') \leq 4\nu
\;\text{whenever}\; \eta \geq 2/\zeta.
\end{align}
\end{corollary}
\begin{proof}
The first statement trivially follows from \cref{lem:app-saddle}(b). For the $\zeta>0$ case, we invoke \cref{lem:app-saddle}(b) with $D = D^*$ and \cref{lem:Slater} with $\widetilde{D} = D'$, $C = \eta$ and $\gamma=2\nu$.
\end{proof}


\subsection{Proof of Theorem \ref{thm:main-detailed}}

We divide the proof into three steps: convergence, regret, and constraint violation. We note that the Slater's condition is only used in the third step.

\xhdr{Step 1 (convergence via no-regret dynamics).} Consider the average play of \PrimalALG and \DualALG: respectively,
    $\pbar_T = \frac{1}{T}\sum_{t\in[T]}\, D_t$
and
    $\lambdabarT = \frac{1}{T}\sum_{t\in[T]}\, \lambda_t$.
 We show that with probability at least $1-O(\delta)$,
\begin{align}\label{eq:saddle-event}
\text{$(\pbar_T, \lambdabarT)$ is a $\nu$-approximate saddle point of the expected Lagrangian $\LagLPeta$},
\end{align}
 with
\begin{align}
\label{eq:nu}
    \nu = \frac{1}{T} \cdot \rbr{\xzedit{\PrimalREG(T)+\DualREG(T)}  +  2 \ConcREG(T,\delta)}.
\end{align}
This step is standard as in~\cite{freund1996game} for a deterministic payoff matrix and  in~\cite{AdvBwK-focs19} for a random payoff matrix. We provide a proof in Appendix~\ref{app:saddle}.
The rest of the analysis conditions on the high-probability event \eqref{eq:saddle-event}.



\xhdr{Step 2 (regret analysis).}
Since $(\pbar_T, \lambdabarT)$ is a $\nu$-approximate saddle point,
Lemma~\ref{lem:app-saddle}(a) implies
\begin{align}
\label{eq:regret-temp}
    r(\pbar_T) \ge r(D^*) - 2\nu = \OPTLP - 2\nu,
\end{align}
where $D^*$ is an optimal solution of~\eqref{eq:LP}.
With this, we obtain the regret bound as follows.
 \begin{align*}
 \OPT-\rew
     &\le T \cdot \OPTLP - {\textstyle \sum_{t\in[T]}}\; r_t(a_t)\\
     & \le T \cdot \OPTLP - T\cdot r(\pbar_T) + \ConcREG(T,\delta)\\
     &\lep{i} 2 T \cdot \nu + \ConcREG(T,\delta)\\
     &\lep{ii} 3\eta' R
 \end{align*}
 where $(i)$ holds by~\refeq{eq:regret-temp}; $(ii)$ holds by definition of $\nu$ in \refeq{eq:nu} and $R = R(T,\delta)$ in~\refeq{eq:regret-joint}.

\xhdr{Step 3 (constraint violations).}
We first note that
\begin{align}
     V_i(T) = \sigma_i\rbr{{\textstyle \sum_{t \in [T]}}\; c_{t,i} - B} &\le  T \sigma_i \left(c_i(\pbar_T) - \tfrac{B}{T}  \right) + \ConcREG(T,\delta)\nonumber\\
     &\le V_{\max}(\pbar_T) + \ConcREG(T,\delta)\label{eq:ViT}.
\end{align}

Thus, it remains to bound $V_{\max}(\pbar_T)$. To this end, recall that we condition on \eqref{eq:saddle-event}, we can now invoke \cref{cor:Vmax}, for $\zeta = 0$ and $\zeta >0$, respectively.



\emph{Case 1: $\zeta = 0$.}
By~\refeq{eq:cons-zeta-0} in \cref{cor:Vmax}, recalling that  $\eta' = \eta\cdot T/B$, we have
\begin{align*}
  V_{\max}(\pbar_T) \le B\cdot \frac{1 + 2\nu}{\eta} = \frac{(1+2\nu) T}{\eta'},
\end{align*}
Hence, by~\refeq{eq:ViT} and the definitions of $\nu$ in \refeq{eq:nu} and $R = R(T,\delta)$ in \refeq{eq:regret-joint}, we have
\begin{align*}
    V_i(T) \le \frac{(1+2\nu) T}{\eta'} + \ConcREG(T,\delta) \le \frac{T}{\eta'} + 2R + \eta' R.
\end{align*}

\emph{Case 2: $\zeta > 0$. }
By \refeq{eq:cons-zeta-p} in \cref{cor:Vmax}, we have
\begin{align*}
    V_{\max}(\pbar_T) \le B\cdot \frac{ 4\nu}{\eta} =  \frac{4\nu T}{\eta'}.
\end{align*}
Hence, using \Cref{eq:ViT,eq:nu,eq:regret-joint} as in Case 1, we have
    $V_i(T) \le  4R + \eta' R$.

\section{\CBwLC via regression oracles}
\label{sec:CB}
In this section, we instantiate the \myAlg framework with \PrimalALG
as \SqCB, a regression-based technique for contextual bandits from
\citet*{regressionCB-icml20}. In particular, we assume access to a
subroutine (``oracle'') for solving the online regression problem, defined below.

\begin{BoxedProblem}{Online regression}
Parameters: $K$ arms, $T$ rounds, context space $\mZ$,  range $[a,b]\subset \R$.\\
In each round $t\in[T]$:
\begin{OneLiners}
  \item[1.] the algorithm outputs a regression function
        $ f_t:\;\mZ\times[K] \to [a,b]$.\\
\myAlgComment{Informally, $f_t(x_t,a_t)$ must approximate the expected score $\E[y_t\mid x_t,a_t]$.}
  \item[2.] adversary chooses
  \regcontext $z_t\in\mZ$, arm $a_t\in [K]$, score $y_t\in[a,b]$,\\ and auxiliary data $\aux_t$ (if any).
  \item[3.] the algorithm receives the new datapoint
    $(z_t,a_t,y_t,\aux_t)$.
\end{OneLiners}
\end{BoxedProblem}




\asedit{(We call $z_t$ a \emph{\regcontext} to distinguish it from contexts in contextual bandits.)}

We assume access to \asedit{an algorithm for online regression} with context
space $\mZ = \mX$, scores $y_t$ equal to rewards (resp., consumption
of a given resource $i$), and no auxiliary data $\aux_t$. It can be an
arbitrary algorithm for this problem, subject to a performance
guarantee stated below in \asedit{\refeq{eq:U-assn}} which asserts that the algorithm can approximate the scores $y_t$ well. We refer to this
algorithm, which we denote by \oracle, as the \emph{online regression oracle}, and invoke it as a
subroutine. Our algorithm for the \CBwLC framework will be efficient whenever the per-round update for the oracle is computationally efficient, \eg the update time does not depend on the time horizon $T$. For simplicity, we use the same oracle for rewards and for each resource $\in[d]$. However, our algorithm and analysis can easily accommodate a different oracle
for each component of the outcome vector.

The quality of the oracle is typically measured in terms of squared regression error, which in turn can be upper-bounded whenever the conditional mean scores are well modeled by a given class $\cF$ of regression functions; this is detailed in \cref{sec:CB-provable,sec:CB-discuss}.

\subsection{Regression-based primal algorithm}

Our primal algorithm, given in \cref{alg:primal}, is parameterized by
an online regression oracle \oracle. We create $d+1$ instances of this
oracle, denoted $\mO_i$, for $i\in[d+1]$, which we apply separately to rewards and to each resource; we
use range $[0,1]$ for rewards and $[-1,+1]$ for
resources. At each
step $t$, given the regression functions $\fhat_{t,i}$ produced by
these oracle instances, \cref{alg:primal} first estimates the expected
Lagrange payoffs in a plug-in fashion (\refeq{eq:Lag-est}). These
estimates are then converted into a distribution over arms in
\refeq{eq:primal-p}; this technique, known as \emph{inverse gap
  weighting} optimally balances exploration and exploitation, as
parameterized by a scalar $\gamma>0$.

\begin{algorithm}[!h]
\Given{$T/B$ ratio, $K$ arms, $d$ resources as per the problem definition;\\
    parameter $\eta\geq 1$ from \myAlg;\\
    online regression oracle \oracle; parameter $\gamma>0$.}
\Init{Instance $\mO_i$ of regression oracle \oracle for each $i\in[d+1]$.}
\myAlgComment{$\fhat_1(x,a)$ and $\fhat_{i+1}(x,a)$ estimate, resp., $r(x,a)$ and $c_i(x,a)$, $i\in [d]$.}
\For{round $t = 1, 2,\, \ldots$ (until stopping)}{
For each oracle $\mO_i$, $i\in[d+1]$: update regression function $\fhat_t = \fhat_{t,i}$.\\
Input context $x_t\in\mX$ and dual distribution
    $\dual = \rbr{\lambda_{t,i}\in [d]}\in\Delta_d$.\\
For each arm $a$, estimate $\E\sbr{ \Lag_t(a,\lambda) \mid x_t}$ with
\begin{align}\label{eq:Lag-est}
\textstyle
    \estL_t(a) := \fhat_{t,1}(x_t,a)
    + \eta\cdot  \sum_{i\in[d]} \sigma_i\cdot\lambda_{t,i}
        \rbr{1  - \tfrac{T}{B}\;\;  \cdot  \fhat_{t,i+1}(x_t,a)}.
\end{align}

Compute distribution over the arms, $p_t\in \Delta_K$, as
\begin{align}\label{eq:primal-p}
p_t(a) = 1/(\; \cnorm_t +
    \gamma\cdot {\textstyle \max_{a'\in[K]}}\;\estL_t(a') - \estL_t(a) \;).
\end{align}
\myAlgComment{$\cnorm_t$ is chosen so that $\sum_a p_t(a) = 1$, via binary search.}
~~Draw arm $a_t$ independently from $p_t$. \\
Output arm $a_t$, input outcome vector
    $\vo_t = (r_t; c_{t,1} \LDOTS c_{t,d})\in [0,1]^{d+1}$.\\
For each oracle $\mO_i$, $i\in[d+1]$: pass a new datapoint $\rbr{x_t,a_t,\,(\vo_t)_i}$.
} 
\caption{Regression-based implementation of \PrimalALG}
\label{alg:primal}
\end{algorithm}

The per-round running time of \PrimalALG is dominated by $d+1$ oracle
calls and $K(d+1)$ evaluations of the regression functions $\fhat_i$ in
\refeq{eq:Lag-est}. For the probabilities in \refeq{eq:primal-p}, it
takes $O(K)$ time to compute the $\max$ expressions, and then
$O(K\log\tfrac{1}{\eps})$ time to binary-search for $\cnorm$ up to
a given accuracy $\eps$.

It is instructive (and essential for the analysis) to formally realize
\PrimalALG as an instantiation of \SqCB \citep{regressionCB-icml20},
a contextual bandit algorithm with makes use of a regression oracle following the protocol described in the prequel. Define the \emph{Lagrange regression} as an online regression problem with data points of the form $(z_t,a_t,y_t,\aux_t)$ for each round $t$, where the \regcontext
    $z_t = (x_t,\dual)$
consists of both the \CBwK context $x_t$ and the dual vector $\dual$,
the score
    $y_t = \Lag_t(a_t,\dual)$
is the Lagrangian payoff as defined by \refeq{eq:Lag-t}, and the auxiliary data
    $\aux_t = \vo_t$
is the outcome vector. \asedit{For the purpose of this problem definition, \regcontexts $z_t$ are adversarially chosen, possibly depending on the history (because the dual vector $\dual$ is generated by the dual algorithm).} The \emph{Lagrange oracle} $\LagOracle$ is an
algorithm for this problem (\ie an online regression oracle) which,
for each round $t$, uses the estimated Lagrangian payoff
\refeq{eq:Lag-t} as a regression function. Thus, \PrimalALG is an
instantiation of \SqCB algorithm equipped, with an oracle
$\LagOracle$ for solving the Lagrange regression problem defined above.

\subsection{Provable guarantees for \cref{alg:primal}}
\label{sec:CB-provable}

\asedit{Formulating our guarantees for \cref{alg:primal} requires some care, as they relies on performance of the regression oracles $\mO_i$, $i\in[d+1]$. (The notions of Lagrange regression/oracle are not needed to state these guarantees; we only invoke them in the analysis in \cref{sec:CB-analysis}).}

Let us formalize the online regression problem faced by a given oracle $\mO_i$, $i\in[d+1]$.
In each round $t$ of this problem, the \regcontext $x_t$ is drawn independently from some fixed distribution,
and the arm $a_t$ is chosen arbitrarily, possibly depending on the history. The score is $y_t = (\vo_t)_i$, the $i$-th component of the realized outcome vector for the $(x_t,a_t)$ pair. Let $f^*_i$ be the ``correct" regression function, given by
\begin{align}\label{eq:correct-regressor}
f^*_i(x,a) = \E\sbr{ (\vo_t)_i \mid x_t=x,\,a_t=a}
    \quad\forall x\in\mX,\,a\in[K].
\end{align}
Following the literature on online regression, we evaluate the
performance of $\cO_i$ in terms of
\emph{squared regression error}:
\begin{align}\label{eq:model-error}
\err_i(\mO_i) := \textstyle \sum_{t\in[T]} \rbr{\fhat_{t,i}(x_t,a_t)-f^*_i(x_t,a_t)}^2,
\quad\forall i\in[d+1].
\end{align}
We rely on a known uniform high-probability upper-bound on these
errors:
\begin{align}\label{eq:U-assn}
\forall \delta\in(0,1)\quad
\exists U_\delta>0\quad
\forall i\in[d+1]\quad
\Pr\sbr{\err_i(\mO_i) \leq U_\delta}\geq 1-\delta.
\end{align}

\asedit{Now we are ready to spell out our primal/dual guarantee:}

\begin{theorem}\label{thm:CB-main}
Suppose \PrimalALG is given by \cref{alg:primal}, invoked with a regression oracle \oracle that satisfies \refeq{eq:U-assn}. Fix an arbitrary  failure probability $\delta\in(0,1)$, let $U = U_{\delta/(d+1)}$, and set the
parameter
    $\gamma = \tfrac{B}{T}\,\sqrt{\tfrac{KT}{d+1}/U}$.
Let \DualALG be the exponential weights algorithm
(``Hedge'') \citep{FS97}. Then \cref{eq:regret-assn,eq:regret-joint} are satisfied with
    $R(T,\delta) = O\rbr{\sqrt{dTU\log (dT/\delta)}}$.
\end{theorem}

This guarantee directly plugs into each of the three  ``theoretical interfaces" of \myAlg (\cref{thm:main}(ab) and \cref{thm:main-BwK}), highlighting the modularity of our approach. In particular, we obtain optimal $\sqrt{T}$ scaling of regret under Slater's condition (and $B\geq \Omega(T))$ and for contextual \BwK, via \cref{thm:main-BwK}. Let us spell out these corollaries for the sake of completeness.


\begin{corollary}\label{cor:CB}
Consider \myAlg with primal and dual algorithms as in \cref{thm:CB-main}, and write $\Phi = dU\log (dT/\delta)$.
Let
    $\regOut := \max_{i\in[d]}\rbr{\OPT-\rew,  V_i(T)}$
denote the outcome-regret.
\begin{itemize}

\item[(a)]
Suppose the LP \eqref{eq:LP} has a $\zeta$-feasible solution, $\zeta\in(0,1)$. Set the parameter to $\eta = 2/\zeta$. Then
$\regOut \leq O\rbr{\nicefrac{T}{B}
                    \cdot\nicefrac{1}{\zeta} \cdot \sqrt{\Phi T}}$
with probability at least $1-O(\delta)$.

\item[(b)]
Suppose the LP \eqref{eq:LP} has a feasible solution. Set the algorithm's parameter as
    $\eta=\frac{B}{T}\,\sqrt{\frac{T}{R(T,\delta)}}$.
Then
    $\regOut \leq O\rbr{\Phi^{1/4}\cdot T^{3/4}}$
with probability at least $1-O(\delta)$.

\item[(c)]
Consider \CBwK with hard-stopping and set $\eta=1$. Then
    $\OPT-\rew \leq O\rbr{\nicefrac{T}{B} \cdot \sqrt{\Phi T}}$
with probability at least $1-O(\delta)$,
and (by definition of hard-stopping) the constraint violations are bounded as $V_i(T)\leq 1$.
\end{itemize}
\end{corollary}

\subsection{Discussion}
\label{sec:CB-discuss}

\noindent\textbf{Generality.}
Online regression algorithms typically restrict themselves to a particular class of regression functions,
        $ \mF \subset \cbr{\mX\times[K] \to \R}$,
so that $f_t\in \mF$ for all rounds $t\in[T]$. Typically, such
algorithms ensure that \refeq{eq:U-assn} holds for a given index
$i\in[d+1]$ whenever a condition known as
\emph{realizability} is satisfied: $f^*_i\in\mF$. Under this
condition, standard algorithms obtain \refeq{eq:U-assn} with
$U_\delta = U_0+\log(2/\delta)$,
where $U_0<\infty$ reflects the intrinsic statistical
capacity of class $\mF$ \citep{Vovk98,azoury2001relative,Vovk06metric,gerchinovitz2013sparsity,rakhlin2014nonparametric}. Standard examples include:
\begin{OneLiners}
\item Finite classes, for which \citet{Vovk98} achieves
  $U_0=\bigoh(\log\abs{\cF})$.
\item Linear classes, where  for a known feature map
    $\phi(x,a)\in\R^{b}$ with $\nrm*{\phi(x,a)}_2\leq{}1$,
regression functions are of the form
    $f(x,a) = \theta\cdot \phi(x,a)$,
for some $\theta\in\R^b$ with $\nrm*{\theta}_2\leq{}1$.
Here, the Vovk-Azoury-Warmuth
  algorithm \citep{Vovk1998competitive,azoury2001relative} achieves
  $U_0\leq{}\bigoh(\xzedit{b}\log(T/\xzedit{b}))$. If $d$ is very large, one could use Online Gradient Descent
  (e.g., \citet{Hazan_book}) and achieve $U_0\leq\bigoh(\sqrt{T})$.
\end{OneLiners}

We emphasize that \refeq{eq:U-assn} can also be ensured via
\emph{approximate} versions of realizability, with the upper bound
$U_\delta$ depending on the approximation quality. The literature on
online regression features various such guarantees, which seamlessly plug into our theorem.
See \citet{regressionCB-icml20} for further background.

\SqCB allows for various extensions to large, structured action
sets. Any such extensions carry over to \myAlg. Essentially, one needs
to efficiently implement computation and sampling of an appropriate
exploration distribution that generalizes \refeq{eq:primal-p}. ``Practical" extensions are known for action sets with linear structure \citep{foster2020adapting,zhu2021making}, and those with Lipschitz-continuity (via uniform discretization) \citep{foster2021statistical}.
More extensions to general action spaces, RL, and beyond are in \citep{foster2021statistical}.

\xhdr{Implementation details.} Several remarks are in order regarding the implementation.

\begin{enumerate}
\item While our theorem sets the parameter $\gamma$ according to the known
  upper bound $U_\delta$, in practice it may be advantageous to treat
  $\gamma$ as a hyperparameter and tune it experimentally.

  \item
  In practice, one could potentially implement the Lagrange oracle by
  applying \oracle to the entire Lagrange payoffs $\Lag_t(a_t,\dual)$
  directly, with $(x_t,\dual)$ as a \regcontext.

  \item
  Instead of computing distribution $p_t$ via \refeq{eq:primal-p} and
  binary search for $\cnorm$, one can do the following
  (cf. \citet{regressionCB-icml20}):  Let
  $b_t = \argmax_{a\in\brk{K}}\estL_t(a)$.  Set
  $p_t(a) = 1/\rbr{K + \gamma\cdot{}(\estL_t(b_t) - \estL_t(a))}$,
  for all $a\neq{}b_t$, and set $p_t(b_t)=1-\sum_{a\neq{}b_t}p_t(a)$.
  This attains the same regret bound (up to absolute constants) as in \cref{thm:SquareCB}.

\item
  In some applications, the outcome vector is determined by an
  observable ``fundamental outcome" of lower dimension. For example,
  in dynamic pricing an algorithm offers an item for sale at a given
  price $p$, and the ``fundamental outcome" is whether there is a
  sale.  The corresponding outcome vector is
  $(p,1)\cdot \indE{\text{sale}}$, \ie a sale brings reward $p$ and
  consumes 1 unit of resource. In such applications, it may be
  advantageous to apply regression directly to the fundamental
  outcomes.
\end{enumerate}


\subsection{Proof of \cref{thm:CB-main}}
\label{sec:CB-analysis}
We incorporate the existing analysis of \SqCB
from \citet{regressionCB-icml20} by applying it to the Lagrange oracle
$\LagOracle$, and restating it in our notation as \cref{thm:SquareCB}.
Define the squared regression error for $\LagOracle$ as
\begin{align}
\err(\LagOracle)
    = \textstyle \sum_{t\in[T]} (\; \estL_t(a_t) - \E\sbr{\Lag_t(a_t,\lambda_t)} \;)^2.
\end{align}
The main guarantee for \SqCB posits a known high-probability
upper-bound on this quantity: 
\begin{align}\label{eq:LagU-assn}
\forall \delta\in(0,1)\quad
\exists \LagU_\delta>0\quad
\Pr\sbr{\err(\LagOracle) \leq \LagU_\delta}\geq 1-\delta.
\end{align}

\begin{theorem}[Implied by \citet*{regressionCB-icml20}]\label{thm:SquareCB}
Consider \cref{alg:primal} with Lagrange oracle that satisfies \refeq{eq:LagU-assn}. Fix   $\delta\in(0,1)$, let $U=\LagU_\delta$ be the upper bound from \refeq{eq:LagU-assn}. Set the parameter
$\gamma = \sqrt{AT/U}$. Then with probability at least $1-O(\delta T)$ we have
\begin{align}\label{eq:thm:SquareCB}
\forall \tau\in [T]\quad
\PrimalREG(\tau) \leq O\rbr{\sqrt{T(U+1)\log (dT/\delta)}}.
\end{align}
\end{theorem}

\begin{remark}
The original guarantee stated in \citet{regressionCB-icml20} is for $\tau=T$ in \cref{thm:SquareCB}. To obtain the guarantee for all $\tau$, as stated, it suffices to replace Freedman inequality in the analysis in \citet{regressionCB-icml20} with its anytime version.
\end{remark}

\begin{remark}
Recall that \regcontexts $z_t = (x_t,\dual)$ in Lagrange regression are treated as  adversarially chosen, because the dual vector $\dual$ is generated by the dual algorithm. In particular, one cannot immediately analyze \PrimalALG via the technique of \citet{regressionCB-bypassing}, which assumes \emph{stochastic} \regcontext arrivals. The analysis from \citet{regressionCB-icml20} that we invoke handles adversarial \regcontext arrivals.
\end{remark}

To complete the proof, it remains to derive \refeq{eq:LagU-assn} from
\refeq{eq:U-assn}, \ie upper-bound $\err(\LagOracle)$
using respective upper bounds for the individual oracles $\mO_i$.
Represent $\err(\LagOracle)$ as
\begin{align*}
\err(\LagOracle)
    = \textstyle \sum_{t\in[T]}\;
        \rbr{ \Phi_t + \eta\cdot\frac{T}{B} \sum_{i\in[d]}\; \lambda_{t,i} \Psi_{t,i} }^2,
\end{align*}
where
    $\Phi_t = \fhat_{t,1}(x_t,a_t) - r(x_t,a_t)$
and
    $\Psi_{t,i}  = c_i(x_t,a_t)-\fhat_{t,i+1}(x_t,a_t)$.
For each round $t$, we have
\begin{align*}
\textstyle
\rbr{ \Phi_t + \eta\cdot\frac{T}{B} \sum_{i\in[d]}\; \lambda_{t,i} \Psi_{t,i} }^2
    &\leq \textstyle 2\,\Phi^2_t +
        2\,(\eta\cdot T/B)^2\, \rbr{\sum_{i\in[d]} \lambda_{t,i} \Psi_{t,i} }^2 \\
    &\leq \textstyle 2\,\Phi^2_t + 2\,(\eta\cdot T/B)^2\, \sum_{i\in[d]}\; \lambda_{t,i} \Psi^2_{t,i},
\end{align*}
where the latter inequality follows from Jensen's inequality. Summing
this up over all rounds $t$,
\begin{align}\label{eq:Lag-ErrBnd}
\err(\LagOracle)
&\leq \textstyle 2(\eta\cdot T/B)^2\,\sum_{i\in[d+1]} \err_i(\mO_i).
\end{align}
The $(\eta\cdot T/B)^2$ scaling is due to the fact that consumption is scaled by $\eta\cdot T/B$ in the Lagrangian, and the error is quadratic. Consequently, \eqref{eq:LagU-assn} holds with
    $\LagU_\delta = (d+1)(\eta\cdot T/B)^2\; U_{\delta/(d+1)}$.
Finally, we plug this $\LagU_\delta$ into~\eqref{eq:thm:SquareCB}, and then normalize $\PrimalREG$ according to~\eqref{eq:regret-joint} to obtain $R(T,\delta)$.


\section{Non-stationary environments}
\label{sec:shifts}
In this section, we generalize the preceding results by allowing the outcome distribution $\outD$ to change over time. In each round $t\in [T]$, the pair $(x_t,\vM_t)$ is drawn independently from some outcome distribution $\outD_t$. The sequence of distributions $(\outD_1 \LDOTS \outD_T)$ is chosen in advance by an adversary (and not revealed to the algorithm). We parameterize our results in terms of the number of switches: rounds $t\geq 2$ such that $\outD_t\neq \outD_{t-1}$; we refer to these as \emph{\enviswitches}. The algorithm does not know when the \enviswitches occur. We refer to such problem instances as the \emph{switching environment}.

We measure regret against a benchmark that chooses the best distribution over policies for each round $t$ separately. In detail, note that each outcome distribution $\outD_t$ defines a version of the linear program \eqref{eq:LP}; call it $\LP_t$. Let $D^*_t\in \Delta_{\allPi}$ be an optimal solution to $\LP_t$, and $\OPTLPt$ be its value. Our benchmark is
    $\OPTpac := \sum_{t\in[T]} \OPTLPt$.
The intuition is that the benchmark would like to pace the resource consumption uniformly over time. We term $\OPTpac$ the \emph{pacing benchmark}. Accordingly, we are interested in the \emph{pacing regret},
\begin{align}\label{eq:pacing-regret}
\regPace := \max_{i\in[d]}\rbr{\OPTpac - \rew(\ALG),\;\viol_i(T)}.
\end{align}

We view the pacing benchmark as a reasonable target for an algorithm that wishes to keep up with a changing environment. However, this benchmark gives up on  ``strategizing for the future", such as underspending now for the sake of overspending later. On the other hand, this property is what allows us to obtain vanishing regret bounds w.r.t. this benchmark. In contrast, the standard benchmarks require moving from regret to approximation ratios once one considers non-stationary environments \citep{AdvBwK-focs19-conf,AdvBwK-focs19}.%
\footnote{This holds even for the special case of only packing constraints and a null arm, and even against the best fixed policy (let alone the best fixed distribution over policies, a more appropriate benchmark for a constrained problem).}

To derive bounds on the pacing regret, we take advantage of the modularity of \LagCBwLC framework and availability of ``advanced"  bandit algorithms that can be ``plugged in" as \PrimalALG and \DualALG. We use algorithms for adversarial bandits that do not make assumptions on the adversary, and yet compete with a benchmark that allows a bounded number of switches (and the same for the full-feedback problem). \asedit{In the ``back-end" of the analysis we invoke \cref{thm:main-detailed}.}

To proceed, we must redefine primal and dual regret to accommodate for switches. First, we \asedit{extend the definition} of primal regret in \refeq{eq:primal-dual-regret-defn} to an arbitrary subset of rounds $\mT\subset[T]$:
\begin{align}
\PrimalREG(\mT) &:= \textstyle
    \sbr{\max_{\pi\in\allPi} \sum_{t\in\mT} \Lag_t(\pi(x_t),\lambda_t )}
    -
    \sum_{t\in\mT} \Lag_t(a_t,\lambda_t ).
\end{align}
Next, an \emph{$S$-switch sequence} is an increasing sequence of rounds
    \asedit{$\seq = \rbr{\tau_j\in [2,T]:\; j\in [S]}$,
with a convention that $\tau_0=1$ and $\tau_{S+1}=T+1$.} The primal regret for $\seq$ is defined as the sum over the intervals between these rounds:
\begin{align}\label{eq:shifting-regret}
\PrimalREG(\seq) &:= \textstyle
    \sum_{j\in[S+1]} \PrimalREG\rbr{[\tau_{j-1},\,\tau_j-1]}.
\end{align}
\asedit{If $\seq$ is the sequence of all \enviswitches, then these are \emph{stationarity intervals}, in the sense that the environment is stochastic throughout
each interval.} For the dual regret,  $\DualREG(\mT)$ and $\DualREG(\seq)$ are defined similarly. We assume a suitable generalization of \refeq{eq:regret-joint}. For every $S\in[T-1]$ and every $S$-switching sequence $\seq$, we assume
\begin{align}\label{eq:regret-assn-switch}
\Pr\sbr{
    \PrimalREG(\seq) \leq \nicefrac{T}{B}\cdot\eta\cdot R_1^S(T,\delta)
    \eqAND
    \DualREG(\seq) \leq \nicefrac{T}{B}\cdot\eta\cdot R_2^S(T,\delta)
} \geq  1-\delta,
\end{align}
for known functions $R_1^S(T,\delta)$ and $R_2^S(T,\delta)$ and  failure probability $\delta\in(0,1)$. Similar to \refeq{eq:regret-joint}, we define ``combined" regret bound \begin{align}\label{eq:regret-joint-switch}
R^S(T,\delta) := R_1^S(T,\delta) + R_2^S(T,\delta) + \sqrt{ST\log (KdT/\delta)},
\end{align}
where the last term accounts for concentration.
\asmargincomment{Defined $R_1^S(T,\delta)$ and $R_2^S(T,\delta)$ to not include $\nicefrac{T}{B}\cdot\eta$ so as to facilitate references such as the ones in \cref{rem:ex-switches}.}

\begin{remark}\label{rem:ex-switches}
\asedit{We do not explicitly assume that $S$ is known. Instead, we note that achieving a particular regret bound \eqref{eq:regret-assn-switch} may require the algorithm to know $S$ or an upper bound thereon.} For the non-contextual setting, implementing \PrimalALG as algorithm EXP3.S (\citet{bandits-exp3}) achieves regret bound
    $R_1^S(T,\delta) = \tO(\sqrt{KST})$
if $S$ is known, and
    $R_1^S(T,\delta) = \tO(S\cdot \sqrt{KT})$
against an unknown $S$. Below, we also obtain
    $R_1^S(T,\delta)\sim \sqrt{ST}$ scaling  for \CBwLC via a variant of \cref{alg:primal} (with known $S$).
For the dual player, the Fixed-Share algorithm
from \cite{herbster1998tracking} achieves
    $R_2^S(T,\delta) = O\rbr{\sqrt{ST\log d}}$
\asedit{when $S$ is known in advance, and     
    $R_2^S(T,\delta) = O\rbr{S\sqrt{T\log d}}$,
when $S$ is not known.}
\asmargincomment{reworded thm; also, removed $S$ factor from the RHS in part (b).}
\end{remark}


\begin{theorem}\label{thm:switch}
Consider \CBwLC with $S$ \enviswitches. Suppose each linear program $\LP_t$, $t\in[T]$ has a
$\zeta$-feasible solution, for some known margin $\zeta\in [0,1)$. Fix $\delta>0$. Consider \myAlg with primal/dual algorithms which satisfy
regret bound \eqref{eq:regret-assn-switch}. Use the notation in \eqref{eq:regret-joint-switch}.
\begin{itemize}

\item[(a)]
If $\zeta>0$, use \myAlg with parameter $\eta = 2/\zeta$. Then with probability at least $1-O(S\delta)$,
\[ \regPace \leq O\rbr{\nicefrac{T}{B}
                    \cdot\nicefrac{1}{\zeta} \cdot R^S(T,\delta)}.\]

\item[(b)]
Use \myAlg with parameter
    $\eta=\frac{B}{T}\,\sqrt{\frac{T}{R^S(T,\delta)}}$.
Then with probability at least $1-O(S\delta)$,
     \[ \regPace \leq O\rbr{\sqrt{T\cdot R^S(T,\delta)}}.\]
\end{itemize}
\end{theorem}

\begin{proof}
\asmargincomment{New proof.}
Let $\seq$ be the $S$-switch sequence that comprises all \enviswitches. Since the problem is stochastic when restricted to each time interval
    $\mathcal{I}_j:=\sbr{t_{j-1},\, t_j-1}$, $j\in [S+1]$,
we can invoke \cref{thm:main-detailed} specialized to this interval, with ``effective" time horizon of $|\mathcal{I}_j|$ (see \cref{rem:shifts-nontrivial}).
Then, we sum up the resulting regret bound over all the intervals. Note that the concentration terms from \cref{thm:main-detailed} sum up as
$ \Lambda:= \sum_{j\in[S+1]} \ConcREG(|\mathcal{I}_j|,\delta) \leq \sqrt{ST\log (KdT/\delta)} $.

Thus, for part (a) we invoke \cref{thm:main-detailed}(a) for each interval $\mathcal{I}_j$, $j\in [S]$. Summing up over the intervals, we obtain, with probability at least $1-O(S\delta)$, that
\begin{align*}
\regPace
    &\leq O\rbr{\PrimalREG(\seq) + \DualREG(\seq) +\Lambda} \\
    &\leq O\rbr{\nicefrac{T}{B}
                    \cdot\nicefrac{1}{\zeta} \cdot R^S(T,\delta)},
\end{align*}
where for the second inequality we invoke the regret bound in  \eqref{eq:regret-assn-switch} for sequence $\seq$.

For part (b), we likewise invoke \cref{thm:main-detailed}(b) for each interval $\mathcal{I}_j$, $j\in [S]$. Summing up over the intervals, we obtain, with probability at least $1-O(S\delta)$, that
\begin{align*}
\forall i\in[d]\quad \viol_i(T)
    &\leq O\rbr{\PrimalREG(\seq) + \DualREG(\seq) + \Lambda + B/\eta } \\
    &\leq O\rbr{\nicefrac{T}{B}
                    \cdot\eta \cdot R^S(T,\delta)+ B/\eta }, \\
\OPTpac - \rew(\ALG)
    &\leq O\rbr{\PrimalREG(\seq) + \DualREG(\seq) + \Lambda} \\
    &\leq O\rbr{\nicefrac{T}{B}
                    \cdot\eta \cdot R^S(T,\delta)}. \\
\text{(Consequently)}\quad \regPace
    &\leq O\rbr{\nicefrac{T}{B}
                    \cdot\eta \cdot R^S(T,\delta)+ B/\eta },
\end{align*}
which is at most $O\rbr{\sqrt{T\cdot R^S(T,\delta)}}$ when $\eta$ is as specified.
\end{proof}

\begin{remark}
Consider the paradigmatic regime when
    $R^S(T,\delta) = \tO(\sqrt{\Psi\cdot ST})$
for some $\Psi$ that does not depend on $T$. Then the regret bounds in \cref{thm:switch} become
    $\regPace\leq \tO\rbr{\nicefrac{T}{B} \cdot\nicefrac{1}{\zeta} \cdot \sqrt{\Psi\cdot ST}}$
for part (a), and
    $\regPace\leq \tO\rbr{(S\Psi)^{\nicefrac{1}{4}}\cdot T^{3/4}}$
for part (b).
\end{remark}

\begin{remark}
\label{rem:shifts-nontrivial}
\ascomment{expanded}
To apply \cref{thm:main-detailed} to every given stationarity interval, the following two features are essential. First, \cref{thm:main-detailed} carries over \emph{as if} the algorithm is run on this interval rather than the full time horizon. This is due to a non-trivial property of \myAlg: it has no memory (and no knowledge of $T$) outside of its primal/dual algorithms. Second,  \cref{thm:main-detailed} invokes \emph{realized} primal (resp., dual)  regret, rather than an upper bound thereon like in \cref{thm:main}. This allows us to leverage the ``aggregate" bound on primal (resp., dual)  regret over the entire switching environment, as per \refeq{eq:shifting-regret}. Using \cref{thm:main} directly would require similar upper bounds for every stationarity interval, which we do not immediately have.%
\footnote{Regret on a given stationarity interval cannot immediately be upper-bounded by the aggregate regret in \eqref{eq:shifting-regret}. This is because per-interval regret can in principle be negative for some (other) stationarity intervals. Besides, this approach would be inefficient even if it does work, resulting in an extra factor of $S$ in the final regret bound.}
 \end{remark}

\begin{remark}\ascomment{expanded}
Even in the special case of packing constraints and hard-stopping (\ie skipping the remaining rounds once some resource is exceeded), it is essential for \cref{thm:switch} that \LagCBwLC continues until the time horizon. Our analysis would not work (even for this special case) if \LagCBwLC is replaced with some algorithm whose guarantees for the stationary environment assume hard-stopping. This is because such guarantees would not bound constraint violations within a given stationarity interval in the switching environment.
\end{remark}

\begin{remark}
\asmargincomment{Using $\mathcal{I}_j$ notation now.}
As an optimization, we may reduce the dependence on $S$ in \cref{thm:switch} by ignoring shorter \enviswitches. Let the sequence $\seq$ and the stationarity intervals $\mathcal{I}_j$  be defined as in the proof.
The time intervals that last $\leq L$ rounds collectively take up
    $\Phi(L) = \sum_{j\in[S]} |\mathcal{I}_j|\cdot \ind{|\mathcal{I}_j|\leq L}$
rounds. We focus on \enviswitches $t_j$ such that $\Phi(|\mathcal{I}_j|)>R$, for some parameter $R$; we call them \emph{$R$-significant}. \cref{thm:switch} can be restated so that $S$ is replaced with the number of $R$-significant \enviswitches, for some $R$ that does not exceed the stated regret bound.
\end{remark}

\xhdr{Primal and dual algorithms.}
We now turn to the task of developing primal algorithms that can be applied within \myAlg in the non-stationary \emph{contextual} setting. To generalize the regression-based machinery from \cref{sec:CB}, define the correct regression function $f^*_{t,i}$ according to  the right-hand side of \refeq{eq:correct-regressor} for each round $t\in[T]$
The estimation error $\err_i(\mO_i)$ is like in \refeq{eq:model-error}, but replacing $f^*_i$ with $f^*_{t,i}$ for each round $t$. In formulae, for each $i\in[d+1]$ we have
\asedit{
\begin{align*}
f^*_{t,i}(x,a) &= \E\sbr{ (\vo_t)_i \mid x_t=x,\,a_t=a}
    \quad\forall x\in\mX,\,a\in[K]. \\
\err_i(\mO_i) &:= \textstyle \sum_{t\in[T]} \rbr{\fhat_{t,i}(x_t,a_t)-f^*_{t,i}(x_t,a_t)}^2.
\end{align*}

\noindent We posit the high-probability error bound \eqref{eq:U-assn}, as in \cref{thm:CB-main}. (A particular error bound of this form may depend on $S$, the number of \enviswitches; achieving it may require the algorithm to know $S$ or an upper bound thereon.)
} 


\begin{theorem}\label{thm:CB-switch}
\asedit{Consider \CBwLC with $S$ \enviswitches such that each linear program $\LP_t$, $t\in[T]$ has a $\zeta$-feasible solution, for some known $\zeta\in [0,1)$.}
Consider \PrimalALG as in as in \cref{thm:CB-main}, with the high-probability error bound $U = U_{\delta/(d+1)}$ defined, for this $S$, via \refeq{eq:U-assn}. Let \DualALG be the Fixed-Share algorithm, as per \cref{rem:ex-switches}.
Then \myAlg with these primal and dual algorithms satisfies the guarantees in \cref{thm:switch}(ab) with
    $R^S(T,\delta) = O\rbr{\sqrt{dTU\log (dT/\delta)}}$.
\end{theorem}

\begin{proof}
\asedit{The full power of \SqCB analysis from \citet*{regressionCB-icml20} implies  \cref{thm:SquareCB} even with \enviswitches, with \refeq{eq:thm:SquareCB} replaced by
\begin{align}\label{eq:thm:SquareCB-switches}
\PrimalREG(\seq) \leq O\rbr{\sqrt{T(U+1)\log (dT/\delta)}},
\end{align}
for any $S$-switch sequence $\seq$ and any $S$.

To bound $\LagU_\delta$ in \refeq{eq:LagU-assn} (which is assumed by \cref{thm:SquareCB}), we observe that \refeq{eq:Lag-ErrBnd} holds (and its proof carries over word-by-word from \cref{sec:CB-analysis}).  Plugging this back into \refeq{eq:thm:SquareCB-switches} and normalizing accordingly, we see that \cref{eq:regret-assn-switch,eq:regret-joint-switch} hold with
 $R^S(T,\delta) = O\rbr{\sqrt{dTU\log (dT/\delta)}}$.}
\end{proof}

\begin{remark}
To obtain \refeq{eq:U-assn}, we assume that each $f^*_{t,i}$ belongs to some known class $\mF$ of regression functions. In particular, if $\mF$ is finite, the regression oracle can be implemented via Vovk's algorithm \citep{Vovk98}, \asedit{applied to the class of all sequences of functions $(f_1 \LDOTS f_T)\in \mF^T$ with at most $S$ switches.}
This achieves \refeq{eq:U-assn} with
  $U_\delta =  O(S\cdot \log\abs{\mF})+\log(2/\delta)$.
Plugging this in, we obtain
$R^S(T,\delta)
= O\rbr{\sqrt{d\log\abs{\mF} \cdot \log(dT/\delta)}}$.
\end{remark}


\section{Conclusions and open questions}
\label{sec:conclusions}
We solve \CBwLC via a Lagrangian approach to handle resource constraints, and a regression-based approach to handle contexts. Our solution emphasizes modularity of both approaches and (essentially) attains optimal regret bounds.


While our main results (\cref{thm:main}(a) and corollaries) assume a known margin $\zeta$ in the Slater condition, it is desirable to recover similar results without knowing $\zeta$ in advance. Several follow-up papers achieve this \citep{guo2024stochastic,castiglioni2024online,bernasconi2024no},  albeit with a worse dependence on the margin.%
\footnote{These papers are follow-up relative to the conference version of our paper, and concurrent work relative to the present version. Outcome-regret scales as $1/\zeta^2$ in \citet{guo2024stochastic,castiglioni2024online} and as $1/\zeta^3$ in \citet{bernasconi2024no}, as compared with $1/\zeta$ dependence in \cref{thm:main}(a) for a known $\zeta$.}
\cite{aggarwal2024no} achieves the same for the special case of auto-bidding, with the same dependence on $\zeta$ as ours.

Given the results in \cref{sec:shifts}, more advanced guarantees for a non-stationary environment
may be within reach. First, one would like to improve dependence on the number of switches, particularly when the changes are of small magnitude. Second, one would like to replace an assumption on the environment (at most $S$ environment-switches) with assumptions on the benchmark.
Similar extensions are known for adversarial bandits (\ie without resources).


\bibliographystyle{plainnat}
\bibliography{bib-abbrv,bib-slivkins,bib-bandits,bib-AGT,refs,refs-Karthik}

\newpage
\appendix

\section{Bandit Convex Optimization with Linear Constraints}
\label{app:BCO}
\newcommand{\BCO}{\term{BCO}}
\newcommand{\BCOwLC}{\term{BCOwLC}}

In this section, we spell out an additional application of \myAlg to bandit convex optimization (\BCO) with linear constraints. We consider \CBwLC with concave rewards, convex consumption of packing resources, and concave consumption of covering resources. Essentially, we follow an application of \LagBwK from \citet[][Section 7.4]{AdvBwK-focs19}, which applies to \BwK with concave rewards and convex resource consumption; we spell out the details for the sake of completeness. Without resource constraints, \BCO has been studied in a long line of work starting from \citet{Bobby-nips04,FlaxmanKM-soda05} and culminating in \citet{Bubeck-colt15,Hazan-nips14,bubeck2017kernel}.

Formally, we consider \emph{Bandit Convex Optimization with Linear Constraints} (\emph{\BCOwLC}), a common generalization of \BwLC and \BCO. We define \BCOwLC as a version of BwK, where the set of arms $\mA$ is a convex subset of $\mathbb{R}^b$. For each round $t$, there is a concave function
        $f_t: \xzedit{\mathcal{A}}\to [0,1]$
and functions
    $g_{t, i}: \xzedit{\mathcal{A}} \rightarrow [-1, 1]$, for each resource $i$, so that
the reward for choosing action $a\in \mA$ in this round is $f_t(a)$ and  consumption of each resource $i$ is $g_{t,i}(a)$. Each function $g_{t,i}$ is convex (resp., concave) if resource $i$ is a packing (resp., covering) resource. In the stochastic environment, the tuple of functions
    $(f_t; g_{t, 1} \LDOTS g_{t,d})$
is sampled independently in each round $t$ from some fixed distribution (which is not known to the algorithm). In the switching environment, there are at most $S$ rounds when that distribution changes.

The primal algorithm \PrimalALG in \myAlg faces an instance of BCO with an adaptive adversary, by definition of Lagrange payoffs~\eqref{eq:Lag-t}. We use a BCO algorithm from \citet{bubeck2017kernel}, which  satisfies the high-probability regret bound against an adaptive adversary. It particular, it obtains \refeq{eq:regret-assn} with
\begin{align}\label{eq:bubeck-regret}
\RegOne = O\rbr{\;\nicefrac{T}{B}\cdot\eta\cdot
    \sqrt{\Phi T }},
\quad\text{where}\quad \Phi = b^{19} \log^{14}(T)\log(1/\delta).
\end{align}
We apply \cref{thm:main} with this primal algorithm, and with Hedge for the dual algorithm.

\begin{corollary}\label{cor:BCO}
Consider \BCOwLC with a convex set of arms $\mA \subset \mathbb{R}^b$ such that LP \eqref{eq:LP} has a $\zeta$-feasible solution for some known $\zeta\in[0,1)$. Suppose the primal algorithm is from \cite{bubeck2017kernel} and the dual algorithm is the exponential weights algorithm
(``Hedge'') \citep{FS97}. Then \cref{eq:regret-assn,eq:regret-joint} are satisfied with
    $R(T,\delta) = O\rbr{\sqrt{\Phi T}}$,
where $\Phi$ is as in \eqref{eq:bubeck-regret}. The guarantees in \cref{thm:main}(ab) apply with this $R(T,\delta)$.  
\end{corollary}


\begin{remark}\label{rem:ext-BCO-infinite}
  This application of the \myAlg framework is admissible because the analysis does not make use of the fact that the action space is finite. In particular, we never take union bounds over actions, and we can replace $\max$ and sums over actions with $\sup$ and integrals.
\end{remark}

\newpage
\section{Details for the Proof of \cref{thm:main-detailed}}
\label{app:aux}

\subsection{Proof of \cref{lem:app-saddle}}
\label{app:lem:app-saddle}

This lemma follows from Lemmas 2-3 in~\cite{agarwal2018reductions}. We prove it here for completeness.

\begin{lemma*}[\cref{lem:app-saddle}, restated]
Let    $(D',\lambda')$ be a $\nu$-approximate saddle point  of $\LagLPeta$.
Then it satisfies the following properties for any feasible solution $D\in\Delta_\Pi$ of LP in~\eqref{eq:LP}:
\begin{itemize}
    \item[(a)] $r(D') \ge r(D) - 2\nu$.
    \item[(b)] $r(D) - r(D') + \frac{\eta}{B}\, \Vmax(D')
    \le 2\nu$.
\end{itemize}
\end{lemma*}

%
 We show the following claim, which will imply \cref{lem:app-saddle}.

   \begin{claim}
   For any feasible $D$, we have
   \begin{align}
   \label{eq:claim}
         r(D') -
         \frac{\eta}{B}\; \Vmax(D')
         + \nu \ge \LagLPeta(D', \lambda')
         \ge  r(D) - \nu.
   \end{align}

   \end{claim}
\begin{proof}
We first show the following upper bound on $\LagLPeta
       (D',\lambda')$. In particular, for any $\lambda \in \Lambda$
   \begin{align*}
       r(D') + \sum_{i\in [d]} \eta \lambda'_i \cdot \sigma_i \left(1  - \frac{T}{B}c_i(D')\right) &= \LagLPeta
       (D',\lambda')\\
       &\le \LagLPeta (D',\lambda) + \nu \\
       &=  r(D') + \sum_{i\in [d]} \eta \cdot \lambda_i \cdot \sigma_i \left(1  - \frac{T}{B}c_i(D')\right) + \nu \\
       &\lep{i} r(D') -
       \eta \max_{i \in [d]}\left[\sigma_i \left(\frac{T}{B}c_i(D')-1  \right)\right]_+
       + \nu,
   \end{align*}
where $(i)$ holds by choosing a specific $\lambda \in \Delta_d$ as follows
\begin{align*}
    \lambda =
\begin{cases}
    0 & \text{if }  \forall i \in [d],  \sigma_i \left(1  - \frac{T}{B}c_i(D')\right) \ge 0 \\
     e_{i^*} & \text{otherwise, where }  i^* = \argmin_{i \in [d]} \sigma_i \left(1  - \frac{T}{B}c_i(D')\right),
\end{cases}
\end{align*}
where $e_i$ denotes the unit vector for the $i$-th dimension,
and $[x]_+ := \max\{x,0\}$.

We also establish a lower bound for $\LagLPeta (D',\lambda')$. Note that for any feasible $D$, since $\lambda' \ge 0$, we have
\begin{align*}
    \LagLPeta(D,\lambda') =  r(D) + \sum_{i\in [d]} \eta \lambda'_i \cdot \sigma_i \left(1  - \frac{T}{B}c_i(D)\right) \ge r(D).
\end{align*}
Moreover, by the approximate saddle point of $(D',\lambda')$,
\begin{align*}
    \LagLPeta(D', \lambda') \ge \LagLPeta(D, \lambda') - \nu.
\end{align*}
Putting them together yields that
     $\LagLPeta(D', \lambda') \ge r(D) - \nu$.
\end{proof}

Now let us use the claim to prove \cref{lem:app-saddle}.

\textbf{Part (a)} follows since the LHS of~\eqref{eq:claim} can be further upper bounded by $r(D') + \nu$.

\textbf{Part (b)} follows by the rearrangement of~\eqref{eq:claim}.


\subsection{Proof of \cref{lem:Slater}}
\label{app:lem:Slater}

Similar results have appeared in~\cite[Theorem 42 and Corollary 44]{efroni2020exploration}, which are variants of results in~\cite[Theorem 3.60 and Theorem 8.42]{beck2017first}, respectively. We provide a proof for completeness. We prove the following, which implies \cref{lem:Slater}.

\begin{lemma}
\label{lem:Slater-app}
Consider the linear program in~\eqref{eq:LP} and suppose Slater's condition holds, i.e.,\xzdelete{i.e.,} some distribution $\hat{D} \in \Delta_{\Pi}$ is $\zeta$-feasible, $\zeta>0$. Then:
\begin{itemize}
    \item[(a)] Let $\lambda^*$ be any optimal dual solution of the dual problem of~\eqref{eq:LP}, then $\norm{\lambda^*}_1 \le \frac{r(D^*) - r(\hat{D})}{\zeta} \le \frac{1}{\zeta}$.
    \item[(b)] Further, suppose the following holds for some $C \ge 2\norm{\lambda^*}_1$, $\widetilde{D} \in \Delta_{\Pi}$ and $\gamma > 0$:
    \begin{align*}
          r(D^*) - r(\widetilde{D}) +
          \tfrac{C}{B}\, \Vmax(\widetilde{D})
          \le \gamma
    \end{align*}
where $D^*$ is an optimal solution of~\eqref{eq:LP}.
    Then
    $\frac{C}{B}\; \Vmax(\widetilde{D}) \leq 2\gamma$.
\end{itemize}
\end{lemma}
\begin{proof}
    Let $q(\lambda) := \max_{D \in \Delta_{\Pi}} \LagLP(D,\lambda)$
    be the dual function. Consider an optimal dual solution $\lambda^* \in \argmin_{\lambda \in R_+^{d}} q(\lambda)$. By strong duality from Slater's condition, we have $q(\lambda^*) = r(D^*) < \infty$.

    To prove part (a), we note that for any optimal dual solution $\lambda^*$, we have
    \begin{align*}
        r(D^*) = q(\lambda^*) \ge r(\hat{D}) + \sum_{i\in [d]}  \lambda_i^* \cdot \sigma_i \left(1  - \frac{T}{B}c_i(\hat{D})\right) \ge r(\hat{D}) + \sum_i  \lambda_i^* \zeta = r(\hat{D}) + \norm{\lambda^*}_1 \zeta,
    \end{align*}
    which gives our first result.

    We turn to prove part (b).
    For $\tau \in \Real^d$, define
    \begin{align*}
        u(\tau) := \max_{D \in \Delta_{\Pi}}\left\{r(D) \mid V'(D) \le -\tau \right\},
    \end{align*}
    where $V'(D) = [V'_1(D),\ldots, V'_d(D)]^{\top}$.
    Note that for any $D \in \Delta_{\Pi}$
    \begin{align*}
        u(0) = r(D^*) = q(\lambda^*) \ge \LagLP(D,\lambda^*)
    \end{align*}
    Hence, we have for any $D$ such that $V'(D) \le -\tau$
    \begin{align*}
        u(0) - \tau^{\top} \lambda^* &\ge \LagLP(D,\lambda^*) - \tau^{\top} \lambda^*\\
        &= r(D) - \sum_i \lambda_i^* V'_i(D) - \tau^{\top} \lambda^*\\
        &\ge r(D).
    \end{align*}
    Maximizing the RHS over all $D$ such that $V'(D) \le -\tau$, gives
    \begin{align*}
         u(0) - \tau^{\top} \lambda^* \ge u(\tau).
    \end{align*}

    Set $\tau = \widetilde{\tau} := - \left[\left[\sigma_1 \left( \frac{T}{B}c_1(\widetilde{D}) - 1  \right)\right]_+,\ldots, \left[\sigma_d \left( \frac{T}{B}c_d(\widetilde{D}) - 1 \right) \right]_+ \right]^{\top}$ in the above inequality, we have
    \begin{align*}
        u(0) - \widetilde{\tau}^{\top} \lambda^* \ge u(\widetilde{\tau}) \ge u(0) = r(D^*) \ge r(\widetilde{D}),
    \end{align*}
    which gives
    \begin{align*}
        r(D^*) - r(\widetilde{D}) \ge \widetilde{\tau}^{\top}\lambda^* \ge - \norm{\lambda^*}_1 \norm{\widetilde{\tau}}_{\infty},
    \end{align*}
    where the last step follows from \xzedit{H\"older's inequality}.

    From this result, we have
    \begin{align*}
        (C - \norm{\lambda^*}_1)\norm{\widetilde{\tau}}_{\infty} &= - \norm{\lambda^*}_1 \norm{\widetilde{\tau}}_{\infty} + C \norm{\widetilde{\tau}}_{\infty}\\
        &\le r(D^*) - r(\widetilde{D}) + C\norm{\widetilde{\tau}}_{\infty} \le \gamma,
    \end{align*}
    where the last step follows the assumption in Lemma~\ref{lem:Slater}. Hence, we finally obtain
    \begin{align*}
         \max_{i \in [d]}\left[\sigma_i \left( \frac{T}{B}c_i(\widetilde{D}) - 1  \right)\right]_+ = \norm{\widetilde{\tau}}_{\infty} \le \frac{\gamma}{C - \norm{\lambda^*}_1} \le \frac{2\gamma}{C},
    \end{align*}
    which follows from $C \ge 2\norm{\lambda^*}_1$, hence finishing the proof.
\end{proof}


\subsection{Convergence to an approximate saddle point}
\label{app:saddle}

We need to prove that \eqref{eq:saddle-event} holds with probability at least $1-O(\delta)$. That is:
\begin{align}
 &\LagLPeta(\pbar_T, \lambdabarT) \ge \LagLPeta(D, \lambdabarT) - \nu,\quad \forall D \in \Delta_{\Pi} \label{eq:sad1}\\
 &\LagLPeta(\pbar_T, \lambdabarT) \le \LagLPeta(\pbar_T, \lambda) + \nu, \quad \forall \lambda \in \Delta_d\label{eq:sad2}.
\end{align}


    To establish~\eqref{eq:sad1}, we note that for any $D \in \Delta_{\Pi}$, with probability $1- O(\delta)$, we have
    \xzmargincomment{replace upper bound by actual regret}
\begin{align*}
    \LagLPeta(D, \lambdabarT) &=  \frac{1}{T} \sum_t \LagLPeta(D, \lambda_t) \\
    & \lep{i} \frac{1}{T} \sum_t \Lag_{t}(D,  \lambda_t) +  \frac{1}{T} \cdot \ConcREG(T,\delta) \\
    & \lep{ii} \frac{1}{T} \sum_t \Lag_{t}(a_t,  \lambda_t) +  \frac{1}{T} \cdot \left(\xzedit{\PrimalREG(T)} +  \ConcREG(T,\delta) \right)\\
    & \lep{iii} \frac{1}{T} \sum_t \Lag_{t}(a_t,  \lambdabarT) + \frac{1}{T}\cdot \left(\xzedit{\DualREG(T) + \PrimalREG(T)}  +  \ConcREG(T,\delta)\right) \\
    & \lep{iv} \LagLPeta(\pbar_T, \lambdabarT) + \frac{1}{T} \cdot \left(\xzedit{\DualREG(T) + \PrimalREG(T)}  +  2 \ConcREG(T,\delta)\right) \\
    & \ep{v} \LagLPeta(\pbar_T, \lambdabarT) + \nu,
\end{align*}
where $(i)$ and $(iv)$ follows from the concentration of
Azuma-Hoeffding inequality; $(ii)$ and $(iii)$ hold by \xzedit{the definitions of primal and dual regrets in Eq.~\eqref{eq:primal-dual-regret-defn}}; $(v)$ holds by definition of $\nu$ in~\eqref{eq:nu}.

We establish~\eqref{eq:sad2} using a similar analysis. In particular, for any $\lambda \in \Delta_d$, we have with probability $1- O(\delta)$
\begin{align*}
    \LagLPeta(\pbar_T, \lambda) &
     \gep{i} \frac{1}{T} \sum_t \Lag_{t}(a_t,  \lambda) -  \frac{1}{T} \cdot \ConcREG(T,\delta) \\
    & \gep{ii} \frac{1}{T} \sum_t \Lag_{t}(a_t, \lambda_t) - \frac{1}{T} \cdot \left(\xzedit{\DualREG(T)} +  \ConcREG(T,\delta) \right)\\
    & \gep{iii} \frac{1}{T} \sum_t \Lag_{t}(\pbar_T, {\lambda_t}) - \frac{1}{T}\cdot \left(\xzedit{\PrimalREG(T) +\DualREG(T)}  +  \ConcREG(T,\delta)\right) \\
    & \gep{iv} \frac{1}{T} \sum_t \LagLPeta(\pbar_T, {\lambda_t}) - \frac{1}{T}\cdot \left(\xzedit{\PrimalREG(T) +\DualREG(T)} +  2 \ConcREG(T,\delta)\right) \\
    & = \LagLPeta(\pbar_T, \lambdabarT) - \frac{1}{T} \cdot \left(\xzedit{\PrimalREG(T) +\DualREG(T)} +  2 \ConcREG(T,\delta)\right) \\
    & \ep{v} \LagLPeta(\pbar_T, \lambdabarT) - \nu,
\end{align*}
where $(i)$ and $(iv)$ follows from the concentration of
Azuma-Hoeffding inequality; $(ii)$ and $(iii)$ hold by \xzedit{the definitions of primal and dual regrets in Eq.~\eqref{eq:primal-dual-regret-defn}}; $(v)$ holds by definition of $\nu$ in~\eqref{eq:nu}.

\section{Zero Constraint Violation}
\label{app:zero}

Let us provide a variant of \cref{thm:main}(a) with zero constraint violation, fleshing out \cref{rem:zero}.

We use \myAlg algorithm with two modifications. First, the budget parameter is scaled down as
    $B' = B(1-\eps)$,
for some $\eps\in (0,\nicefrac{1}{2}]$.
Second, for each \emph{covering} resource $i\in [d]$ and each round $t$, the consumption reported back to \myAlg is scaled down as
$c'_{t,i} = c_{t,i} - 2 \eps B / T$.
\footnote{Note that $c'_{t,i}\in [-2,1]$, whereas the original model has $c_{t,i}\in [-1,1]$. However, the analysis leading to \cref{thm:main}(a) carries over without modifications, and it is only the constants in $O(\cdot)$ that change slightly.}
For packing resources $i$, reported consumption stays the same: $c'_{t,i}=c_{t,i}$.
The modified algorithm, called \myAlgResc, has two parameters: $\eta\geq 1$ as before and $\eps\in (0,\nicefrac{1}{2}]$ for rescaling. Effectively, the modifications have made all resources slightly more constrained.

\begin{theorem}\label{thm:zero}
Suppose some solution for LP~\eqref{eq:LP} is $\zeta$-feasible, for a known margin $\zeta>0$. Fix some $\delta>0$ and consider the setup in \cref{eq:primal-dual-regret-defn,eq:regret-assn,eq:regret-joint}
with ``combined" regret bound $R(T,\delta)$. Consider algorithm  \myAlgResc with parameters $\eta = 4/\zeta$ and
    $\eps = 16 \cdot \frac{TR(T,\delta)}{\zeta B^2}$.
Assume the budget is large enough so that
    $\epsilon \le \zeta/2$.
Then with probability at least $1-O(\delta)$ we have $\min_{i\in[d]} V_i(T) \le 0$  and
\begin{align} \label{eq:thm:zero}
\OPT-\rew
    \leq O\rbr{ \rbr{\nicefrac{T}{B} \cdot\nicefrac{1}{\zeta}}^2 \cdot R(T,\delta)}.
\end{align}
\end{theorem}

\begin{remark}
The regret bound is the same as in \cref{thm:main}(a), up to the factor of
    $\nicefrac{T}{B} \cdot\nicefrac{1}{\zeta}$
(and the paradigmatic case is that this factor is an absolute constant).
\end{remark}

In the rest of this appendix, we prove \cref{thm:zero}. An execution of  \myAlgResc on the original problem instance $\mI$ can be interpreted as an execution of \myAlg on a modified problem instance $\mI'$ which has budget $B'$ and realized consumptions $c'_{t,i}$ as defined above, and the same realized rewards as in $\mI$. Note that $\mI'$ is slightly more constrained in each resource compared to $\mI$, \ie a slightly ``harder'' instance.

Let us write down a version of the LP \eqref{eq:LP} for the modified instance $\mI'$:
\begin{equation}
\label{eq:LP-new}
\begin{array}{ll}
\text{maximize}
    &r(D)\\
\text{subject to} &D \in \Delta_{\allPi}  \\
    & V'_i(D) := \sigma_i\rbr{ T\cdot c'_i(D) - B'}
    \leq 0
    \qquad \forall i \in [d],
\end{array}
\end{equation}
where $c'_i(D) := \E\sbr{c'_{t,i}(\pi(x_t))}$  and
the expectation is over $\pi\sim D$ and  $(x_t,\vM_t)\sim\outD$. Let $\OPTLP'$ be the value of this LP, and recall that $\OPT' = T\cdot \OPTLP'$. Note that for each resource $i\in[d]$,
\begin{align}\label{eq:app-zero-Vprime}
V'_i(D)=  \sigma_i\rbr{ T\cdot c_i(D) - B} +  \eps B = V_i(D)+\eps B
\qquad \forall D\in \Delta_{\allPi}.
\end{align}
(Indeed, this holds for both packing and covering resources $i$.)

We claim that $\mI'$ satisfies Slater condition with margin $\zeta' = \zeta/2$. Take $\hat{D}$ be the $\zeta$-feasible solution to $\cI$ guaranteed by the theorem statement. Then $V_i(\hat{D})\leq -\zeta B$ for each resource $i\in[d]$. So by \refeq{eq:app-zero-Vprime} we have
    $V'_i(D) \leq \eps B - \zeta B \leq \zeta B/2$
because $\eps\leq \zeta/2$ by assumption and $B' = B(1-\eps)$. Claim proved.


Thus, we can now invoke \cref{thm:main}(a) for \myAlg and the modified problem instance $\mI'$. So, with probability at least $1-O(\delta)$, we have
\begin{align} \label{eq:app-zero-invoke}
\max_{i\in[d]}\rbr{\OPT'-\rew,\; V'_i(T)}
    \leq O\rbr{ \nicefrac{T}{B'}
                    \cdot\nicefrac{1}{\zeta'}\cdot R(T,\delta)}
    \leq O\rbr{ \nicefrac{T}{B}
                    \cdot\nicefrac{1}{\zeta}\cdot R(T,\delta) },
\end{align}
where $\OPT'$ and $V'_i(T)$ are, resp., the benchmark \eqref{eq:prelims-opt} and the constraint violation \eqref{eq:prelims-constraint} for the modified problem instance. It remains to ``massage" this guarantee to obtain regret bound \eqref{eq:thm:zero} and no constraint violations for the original problem instance. In what follows, let us condition on the event in \eqref{eq:app-zero-invoke}.




To analyze regret, we bound the difference
    $\OPTLP- \OPTLP'$.
We construct a feasible solution to instance $\mI'$ via a mixture of $D^*$, the optimal solution for the original LP \eqref{eq:LP}, and the $\zeta$-feasible solution $\hat{D}$ to $\cI$. Specifically, consider
   $ {D'} := \rbr{1-\nicefrac{\eps}{\zeta}} D^* + \nicefrac{\eps}{\zeta}\,\hat{D}$.
Hence, we have ${V}_i^{\prime}({D'}) \le 0$ for all $i \in [d]$, i.e., it is a feasible solution of the new LP \eqref{eq:LP-new}. Consequently,
\[ \OPTLP'  \ge r(D')\geq \rbr{1-\nicefrac{\eps}{\zeta}} \OPTLP, \]
so $ \Delta := \OPTLP- \OPTLP' \leq \eps/\zeta$.
Finally,
\begin{align*}
\OPT - \rew &\le  {\OPT'} - \rew + T \Delta \\
    & = O\rbr{ \nicefrac{T}{B}
                    \cdot\nicefrac{1}{\zeta}\cdot R(T,\delta) + \nicefrac{T^2}{B^2}
                    \cdot\nicefrac{1}{\zeta^2}\cdot R(T,\delta)},
\end{align*}
where we plugged in \eqref{eq:app-zero-invoke} and the definition of $\eps$.

To analyze constraint violation, use \eqref{eq:app-zero-invoke} in a more explicit version from \cref{thm:main-detailed}, we have
\begin{align*}
    {V}_i^{\prime}(T)
    &\le  4 \cdot \nicefrac{T}{B'}
                    \cdot\nicefrac{1}{\zeta'}\cdot R(T,\delta) \\
    &\le 16 \cdot \nicefrac{T}{B}
                    \cdot\nicefrac{1}{\zeta}\cdot R(T,\delta)
                    = \eps B. \\
     V_i(T)
        &= \sigma_i\rbr{{\textstyle \sum_{t \in [T]}}\; c_{t,i} - B} \\
        &=  \sigma_i\rbr{{\textstyle \sum_{t \in [T]}}\; c'_{t,i} - B'} - B \eps
        \leq 0.
\end{align*}

\end{document}
